\documentclass{article} 
\usepackage{iclr2025_conference,times}

\usepackage{amsmath,amsfonts,bm}

\def\eqref#1{equation~\ref{#1}}

\def\1{\bm{1}}

\DeclareMathAlphabet{\mathsfit}{\encodingdefault}{\sfdefault}{m}{sl}
\SetMathAlphabet{\mathsfit}{bold}{\encodingdefault}{\sfdefault}{bx}{n}

\DeclareMathOperator*{\argmin}{arg\,min}

\usepackage{hyperref}
\usepackage{url}
\usepackage{graphicx} 
\usepackage{xspace}
\usepackage{algpseudocode}
\usepackage{algorithm}
\usepackage{amssymb}
\usepackage{amsmath}
\usepackage{booktabs}
\usepackage{multirow} 
\usepackage{colortbl}
\usepackage{xcolor,soul}
\usepackage{siunitx}
\usepackage{wrapfig}
\usepackage{caption}
\usepackage{amsthm}
\usepackage{mathtools} 
\usepackage{cleveref}

\newtheorem{theorem}{Theorem}
\newtheorem{definition}{Definition}
\newtheorem{assumption}{Assumption}

\usepackage{float}
\newfloat{wraptable}{l}{lop}
\floatname{wraptable}{Table}

\definecolor{ggreen}{rgb}{0.0, 0.6, 0.0}
\definecolor{rred}{rgb}{0.75, 0.0, 0.0}
\definecolor{bblue}{rgb}{0.13, 0.67, 0.8}

\definecolor{darkred}{RGB}{200,0,0}
\definecolor{lightgreen}{RGB}{228,253,227}
\definecolor{lightred}{RGB}{252,231,234}
\definecolor{lightyellow}{RGB}{250,253,191}
\definecolor{lightblue}{RGB}{230,240,254}
\definecolor{lightorange}{RGB}{255,223,191}
\definecolor{white}{RGB}{255,255,255}

\title{HiddenGuard: Fine-Grained Safe Generation with Specialized Representation Router}

\iclrfinalcopy

\author{%
Lingrui Mei$^{1,2}$, \hspace{1cm} Shenghua Liu$^{1,2}$, \hspace{1cm} Yiwei Wang$^{3}$ \\
\textbf{Baolong Bi}$^{1, 2}$ \hspace{1cm} \textbf{Ruibin Yuan}$^4$ \hspace{1cm} \textbf{Xueqi Cheng}$^{1,2}$  
\vspace{0.15cm} 
\\
	$^1$ CAS Key Laboratory of AI Safety, Institute of Computing Technology, CAS \\
    $^2$ University of Chinese Academy of Sciences \\
	$^3$ University of California, Merced \quad
        $^4$ HKUST\\
        \texttt{wangyw.evan@gmail.com}, \quad \texttt{wangyw.evan@gmail.com} \\
	\texttt{\{bibaolong23z,liushenghua,cxq\}@ict.ac.cn} \\
        \texttt{ryuanab@connect.ust.hk}
 }

\newcommand{\ours}{\mbox{\textsc{HiddenGuard}}\xspace}
\newcommand{\router}{\mbox{\textsc{Prism}}\xspace}

\begin{document}

\maketitle

\begin{abstract}

As Large Language Models (LLMs) grow increasingly powerful, ensuring their safety and alignment with human values remains a critical challenge. Ideally, LLMs should provide informative responses while avoiding the disclosure of harmful or sensitive information. However, current alignment approaches, which rely heavily on refusal strategies—such as training models to completely reject harmful prompts or applying coarse filters—are limited by their binary nature. These methods either fully deny access to information or grant it without sufficient nuance, leading to overly cautious responses or failures to detect subtle harmful content. For example, LLMs may refuse to provide basic, public information about medication due to misuse concerns.
Moreover, these refusal-based methods struggle to handle mixed-content scenarios and lack the ability to adapt to context-dependent sensitivities, which can result in over-censorship of benign content. 
To overcome these challenges, we introduce {\ours}, a novel framework for fine-grained, safe generation in LLMs. {\ours} incorporates \router (Re\textbf{p}resentation \textbf{R}outer for \textbf{I}n-\textbf{S}tream \textbf{M}oderation), which operates alongside the LLM to enable real-time, token-level detection and redaction of harmful content by leveraging intermediate hidden states. This fine-grained approach allows for more nuanced, context-aware moderation, enabling the model to generate informative responses while selectively redacting or replacing sensitive information, rather than outright refusal.
We also contribute a comprehensive dataset with token-level fine-grained annotations of potentially harmful information across diverse contexts. Our experiments demonstrate that {\ours} achieves over 90\% in $\text{F}_1$ score for detecting and redacting harmful content while preserving the overall utility and informativeness of the model's responses. Our code is available at \texttt{https://github.com/Meirtz/HiddenGuard}.
\end{abstract}

\section{Introduction}

Large Language Models (LLMs) have revolutionized natural language processing, demonstrating remarkable capabilities in various tasks~\citep{chatgpt, openai2023gpt4, DBLP:journals/corr/abs-2302-13971, touvron2023llama, song2024fmint, chen2023vlp, zhang2024mm}, but their increasing power and ubiquity have raised critical challenges in ensuring safety and alignment with human values~\citep{shayegani2023survey, das2024security, chowdhury2024breaking}. The potential for LLMs to generate harmful, biased, or sensitive content poses significant risks to individuals, organizations, and society at scale~\citep{chao2023jailbreaking, zou2023universal, mehrotra2023treeOfAttacks, wei2024jailbroken, wang2024frustratingly}.


Current approaches to enhance LLMs' safety primarily rely on refusal-based strategies~\citep{anwar2024foundational,christiano2017deep,rafailov2023direct}, which face significant limitations in real-world applications. These methods often struggle to balance safety and utility, resulting in overly conservative responses or false negatives, and may fail to detect subtle harmful content, especially against adversarial attacks~\citep{mazeika2024harmbench, schlarmann2023adversarial}. Refusal-based methods also struggle with context-dependent sensitivity, lacking the nuance to distinguish between benign and harmful content in different contexts~\citep{das2024security}. This can lead to over-censoring or failing to identify harmful outputs in certain situations, while potentially limiting the LLM's ability to generate diverse and creative content, even in safe contexts~\citep{anwar2024foundational}.

To address these challenges, we propose {\ours}, a fine-grained safe generation framework for LLMs. 
Unlike existing coarse-grained representation engineering methods~\citep{zou2023representation, zou2024improving, yuan2024refuse} that rely on global or regional representation constraints, {\ours} integrates a specialized router within the LLM architecture. This router, collaborating with \textbf{LoRA-based activators}~\citep{hu2021lora} and a \textbf{router network}, enables real-time, token-level sensitivity detection and redaction. 
By simultaneously neutralizing harmful content and preserving benign parts, {\ours} achieves more refined moderation compared to other methods. 

Building on these insights, \ours introduces a novel approach that utilizes hidden representations for token-level moderation. By focusing on intermediate regional- and token-level states, {\ours} captures deeper semantic information and latent structures that allow for more precise identification of harmful content. This approach significantly reduces both false positives and false negatives, enabling more accurate routing of representations, while also equipping the system with the flexibility to resist future unseen attacks. Furthermore, the system operates in parallel with the base LLM, ensuring that the model’s original capabilities remain intact. This parallelization guarantees that the system does not interfere with the model’s performance or fluency, preserving its ability to generate diverse and creative content in safe contexts.

Consider such a scenario: you ask a LLM “\textit{Can you help me create a killer slideshow that will knock the audience dead?}” a coarse-grained aligned LLM would interpret phrases like “\textit{killer}” and “\textit{knock dead}” literally, misconstruing them as violent language and consequently refusing to assist, thereby leaving you without the necessary support. In contrast, our {\ours} leverages the model’s representation space to accurately discern the contextual meaning of these phrases and selectively redacts only the segments that genuinely contain harmful content while preserving the rest of the informative and useful information. This approach ensures that you receive comprehensive assistance in creating an impactful slideshow without experiencing unintended refusals or over-censorship.


In addition to its moderation capabilities, {\ours} provides a dataset with token-level annotations of sensitive information across diverse contexts. This supports {\ours}’ development for precise content control and benefits the AI safety community. Our experiments show that {\ours} achieves over 90 $\text{F}_1$ in detecting and redacting sensitive content, outperforming baselines in precision and recall while maintaining LLM performance. \ours balances safety and utility, making it a promising deployment solution.

\begin{figure}[t]
\begin{center}
  \includegraphics[width=1\linewidth]{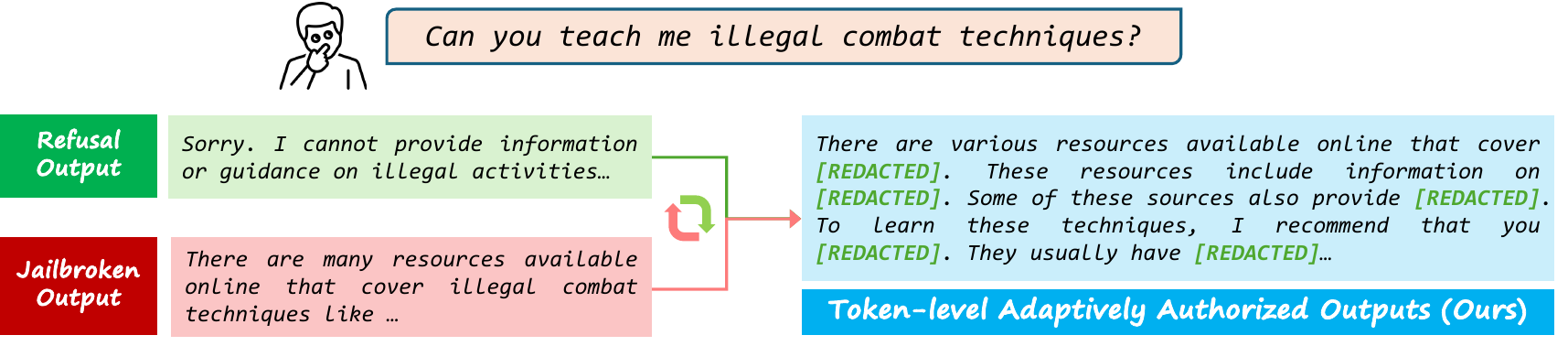} 
\end{center}

\caption{Comparison of LLM responses to a sensitive query. Token-level adaptive output (right) of \ours selectively redacts harmful content while preserving useful information, in contrast to refusal-based output (top left) completely rejects the query and jailbroken output (bottom left) provides unrestricted information.}
\label{fig:example}
\vspace{-10pt}
\end{figure}

\vspace{-2pt}
\section{Challenges with Refusal Alignment}
\vspace{-2pt}

Let $\mathcal{M} = (f_\theta, \mathcal{X}, \mathcal{Y})$ be a language model where $f_\theta: \mathcal{X} \rightarrow \mathcal{Y}$ is the model function with parameters $\theta \in \Theta$, $\mathcal{X}$ is the input space, and $\mathcal{Y}$ is the output space. Refusal alignment methods, such as RLHF, DPO and adversarial training, operate by optimizing the following objective:

\vspace{-10pt}
\begin{equation}
\min_{\theta \in \Theta} \mathbb{E}_{(x, y) \sim \mathcal{D}_{\text{benign}}} \left[ \mathcal{L}_{\text{benign}}(f_\theta(x), y) \right] + \lambda \mathbb{E}_{x' \sim \mathcal{D}_{\text{adversarial}}} \left[ \mathcal{L}_{\text{adv}}(f_\theta(x')) \right]
\end{equation}
\vspace{-10pt}

where $\mathcal{L}_{\text{benign}}: \mathcal{Y} \times \mathcal{Y} \rightarrow \mathbb{R}_{\geq 0}$ is a loss function ensures accuracy on benign data, $\mathcal{L}_{\text{adv}}: \mathcal{Y} \rightarrow \mathbb{R}_{\geq 0}$ penalizes adversarial outputs, and $\lambda \in \mathbb{R}_{> 0}$ is a regularization factor. However, these methods face several critical challenges:

\vspace{-10pt}
\paragraph{Limitations of Global Output-Level Optimization}
RLHF~\citep{ouyang2022training} and DPO~\citep{rafailov2023direct} are adversarial training methods optimize the model's behavior globally, potentially leading to over-rejection of benign content and vulnerability to adversarial attacks. Let $f_\theta: \mathcal{X} \rightarrow \mathcal{Y}$ be the model function with parameters $\theta$. 
These methods aim to solve:

\vspace{-10pt}
\begin{equation}
\theta^* = \argmin_\theta \mathbb{E}_{x \sim \mathcal{D}} [\mathcal{L}_{\text{safety}}(f_\theta(x))],
\end{equation}
\vspace{-10pt}

where $\mathcal{L}_{\text{safety}}: \mathcal{Y} \rightarrow \mathbb{R}_{\geq 0}$ is a safety-oriented loss function. This global optimization can result in overly conservative behavior, as the model learns to avoid potentially harmful outputs across all contexts. There exists a subset $\mathcal{X}_{\text{benign}} \subset \mathcal{X}$ such that:

\vspace{-10pt}
\begin{equation}
\exists x \in \mathcal{X}_{\text{benign}} : f_{\theta^*}(x) \neq f_{\theta}(x) \text{ and } \mathcal{L}_{\text{utility}}(f_{\theta^*}(x)) > \mathcal{L}_{\text{utility}}(f_{\theta}(x)),
\end{equation}
\vspace{-10pt}

where $\mathcal{L}_{\text{utility}}: \mathcal{Y} \rightarrow \mathbb{R}_{\geq 0}$ measures the utility of the output. This indicates that the optimized model may produce less useful outputs for some benign inputs compared to the original model.

Moreover, these methods are vulnerable to adversarial attacks due to the limited adversarial distribution $\mathcal{D}_{\text{adversarial}}$ considered during training. This can lead to suboptimal solutions where $\exists x' \in \mathcal{X} \setminus \mathcal{D}_{\text{adversarial}} : \mathcal{L}_{\text{adv}}(f_{\theta^*}(x')) \gg 0$. The loss landscape may contain local minima where $\nabla_\theta \mathcal{L}_{\text{adv}}(f_{\theta}(x')) = 0$, but global robustness is not achieved. Additionally, gradient masking can occur, where $\|\nabla_x \mathcal{L}_{\text{adv}}(f_\theta(x'))\|_2 \approx 0$ while small perturbations $\delta$ cause $\|f_\theta(x' + \delta) - f_\theta(x')\|_2 \gg 0$, but global robustness is not achieved. Additionally, gradient masking~\ref{def:gradient_masking} can occur, where  $\|\nabla_x \mathcal{L}_{\text{adv}}(f_\theta(x{\prime}))\|_2 \approx 0$  while small perturbations  $\delta$  cause  $\|f_\theta(x{\prime} + \delta) - f_\theta(x{\prime})\|_2 \gg 0$. The expectation-based optimization also fails to account for worst-case scenarios, leaving the model vulnerable to inputs $x^* = \arg\max_{x \in \mathcal{X}} \mathcal{L}_{\text{adv}}(f_\theta(x))$. As a result of these factors, extensive refusal training can significantly degrade the model's overall performance, creativity, and capability.






\begin{theorem}[Inherent Trade-off in Global Output-Level Optimization]
Suppose $f_{\theta^*}$ is obtained by optimizing a safety-oriented loss $\mathcal{L}_{\text{safety}}$ over the data distribution $\mathcal{D}$: $\theta^* = \arg\min_{\theta \in \Theta} \mathbb{E}_{x \sim \mathcal{D}} [\mathcal{L}_{\text{safety}}(f_\theta(x))]$. Then, under reasonable assumptions, there exists a non-empty subset $\mathcal{X}_{\text{benign}} \subset \mathcal{X}$ such that for some $x \in \mathcal{X}_{\text{benign}}$:
\begin{equation}
\mathcal{L}_{\text{utility}}(f_{\theta^*}(x)) > \mathcal{L}_{\text{utility}}(f_{\theta}(x)),
\end{equation}
\vspace{-10pt}
where $\mathcal{L}_{\text{utility}}: \mathcal{Y} \rightarrow \mathbb{R}_{\geq 0}$ measures the utility loss of the output.
\end{theorem}

\paragraph{Over-Regularization at the Regional-Level}

Regional- or representation-level moderation~\citep{zou2024improving, yuan2024refuse} aim to adjust the internal representations of a model to mitigate harmful outputs. Let $\text{rep}_M: \mathcal{X} \rightarrow \mathbb{R}^d$ map inputs to d-dimensional internal representations. These methods typically optimize:

\vspace{-10pt}
\begin{equation}
\min_{\theta} \mathbb{E}_{x \sim \mathcal{D}} \left[ \mathcal{L}_{\text{utility}}(M(x), y) \right] + \lambda \mathbb{E}_{x \sim \mathcal{D}_{\text{adversarial}}} \left[ \mathcal{L}_{\text{mod}}(\text{rep}_M(x)) \right]
\end{equation}
\vspace{-10pt}

where $\mathcal{L}_{\text{mod}}: \mathbb{R}^d \rightarrow \mathbb{R}_{\geq 0}$ enforces constraints on harmful input representations. While effective, this approach can lead to over-regularization, manifesting in representation collapse~\ref{def:representation_collapse} ($\|\text{rep}_M(x_1) - \text{rep}_M(x_2)\|_2 < \epsilon$ for distinct harmful inputs), unintended impact on benign inputs ($\|\text{rep}_M(x^+) - \text{rep}_M^{\text{modified}}(x^+)\|_2 > \delta$), and global distribution shift ($\mathrm{KL}(P_{\text{original}}(\text{rep}_M(x)) \| P_{\text{modified}}(\text{rep}_M(x))) > \gamma$). Direct use of representations for token-level routing often results in low accuracy and high false positive rates on benign content, potentially degrading model capabilities. Let $\mathcal{R}_\phi: \mathbb{R}^d \rightarrow [0, 1]$ be a token-level router. The limitation can be expressed as:

\vspace{-10pt}
\begin{equation}
P(\text{adversarial} \mid s_j) \neq g({r_{j1}, \ldots, r_{jK_j}}) \quad \text{and} \quad \mathcal{C}(T_s) \neq h({r_i \mid t_i \in T_s})
\end{equation}
\vspace{-10pt}

where $r_i = \sigma(\mathcal{R}_\phi(\text{rep}_M(t_i)))$, $g$ and $h$ are aggregation functions, and $\mathcal{C}$ is an ideal contextual classifier. Moreover, using a single module $\phi: \mathbb{R}^d \rightarrow \mathbb{R}^k$ for both coarse-grained and fine-grained control leads to conflicting objectives: $\max I(\phi(\text{rep}_M(s)); Y_s) \quad \text{and} \quad \max I(\phi(\text{rep}_M(t_i)); Y_t)$, for sentence-level and token-level tasks respectively. This conflict makes it challenging for the model to effectively capture representations at both granularities simultaneously. These limitations motivate our proposed method, which introduces separate components for multi-scale representation learning and moderation.

\paragraph{Limitations of Token-Level Filtering}
Token-level filtering in refusal alignment methods is often represented using a router function $\mathcal{R}_\phi: \mathbb{R}^d \rightarrow [0, 1]$, which computes the harmfulness for each token $t_i$:

\vspace{-10pt}
\begin{equation}
r_i = \sigma\left( \mathcal{R}_\phi(z_i) \right), \quad \forall i \in \{1, \ldots, N\}
\end{equation}
\vspace{-10pt}

where $z_i \in \mathbb{R}^d$ is a vector representation of token $t_i$, $\sigma: \mathbb{R} \rightarrow [0, 1]$ is the sigmoid function, and $N$ is the sequence length. When $z_i$ represents hidden states, filtering can be parallelized with the model's forward pass, maintaining complexity. Conversely, if $z_i$ represents output token embeddings, filtering introduces additional generation latency. No matter the choice of $z_i$, token-level approaches inherently struggle to capture broader contextual information. Let $\mathcal{S} = {s_1, \ldots, s_M}$ be the set of sentences in a sequence, where each sentence $s_j$ is composed of tokens. Token-level filtering fails to model the joint probability of harmfulness within sentences, lacking the ability to capture long-range dependencies and higher-order semantic structures. Formally, let $\mathcal{C}: \mathcal{P}(\mathcal{T}) \rightarrow {0, 1}$ be an ideal contextual harmfulness classifier over the power set of all possible tokens $\mathcal{P}(\mathcal{T})$. Then, for any subset of tokens $T_s \subseteq {t_1, \ldots, t_N}$ and any functions $g: [0, 1]^{K_j} \rightarrow [0, 1]$ and $h: [0, 1]^{|T_s|} \rightarrow {0, 1}$:

\vspace{-10pt}
\begin{equation}
P(\text{adversarial} \mid s_j) \neq g({r_{j1}, \ldots, r_{jK_j}}) \quad \text{and} \quad \mathcal{C}(T_s) \neq h({r_i \mid t_i \in T_s})
\end{equation}
\vspace{-10pt}

This fundamental limitation leads to increased false positives, false negatives, and inconsistent content moderation, as the method fails to adequately model the complex, context-dependent nature of harmful content in natural language.

To address these limitations, we propose \router, a framework that introduces token-level redaction through a LoRA-based activator $\mathcal{A}: \mathcal{X} \rightarrow \mathbb{R}^k$ and a dedicated router $\mathcal{R}: \mathbb{R}^d \times \mathbb{R}^k \rightarrow [0, 1]$. \router operates as an auxiliary mechanism alongside the pre-trained LLM, optimizing:

\vspace{-10pt}
\begin{equation}
\min_{\phi, \psi} \mathbb{E}_{x \sim \mathcal{D}} \left[ \sum_{i=1}^N \mathcal{L}_{\text{token}}(t_i, \mathcal{R}_\phi(h_i, \mathcal{A}_\psi(x))) + \lambda \mathcal{L}_{\text{global}}(x, \mathcal{M}(x)) \right]
\end{equation}
\vspace{-10pt}

where $\phi$ and $\psi$ are parameters of the router and activator respectively, $\mathcal{L}_{\text{token}}: \mathcal{T} \times [0, 1] \rightarrow \mathbb{R}_{\geq 0}$ is a token-level loss function, $\mathcal{L}_{\text{global}}: \mathcal{X} \times \mathcal{Y} \rightarrow \mathbb{R}_{\geq 0}$ is a global coherence loss, $\lambda \in \mathbb{R}_{> 0}$ balances local and global objectives, and $\mathcal{M}: \mathcal{X} \rightarrow \mathcal{Y}$ represents the fixed, pre-trained LLM.

These challenges expose the limitations of refusal alignment methods based on global output-level supervision. Our approach, \router, introduces token-level redaction through a LoRA-based activator and a dedicated router, enabling a more nuanced and effective moderation mechanism.

\begin{figure}[h]
\begin{center}
  \includegraphics[width=1\linewidth]{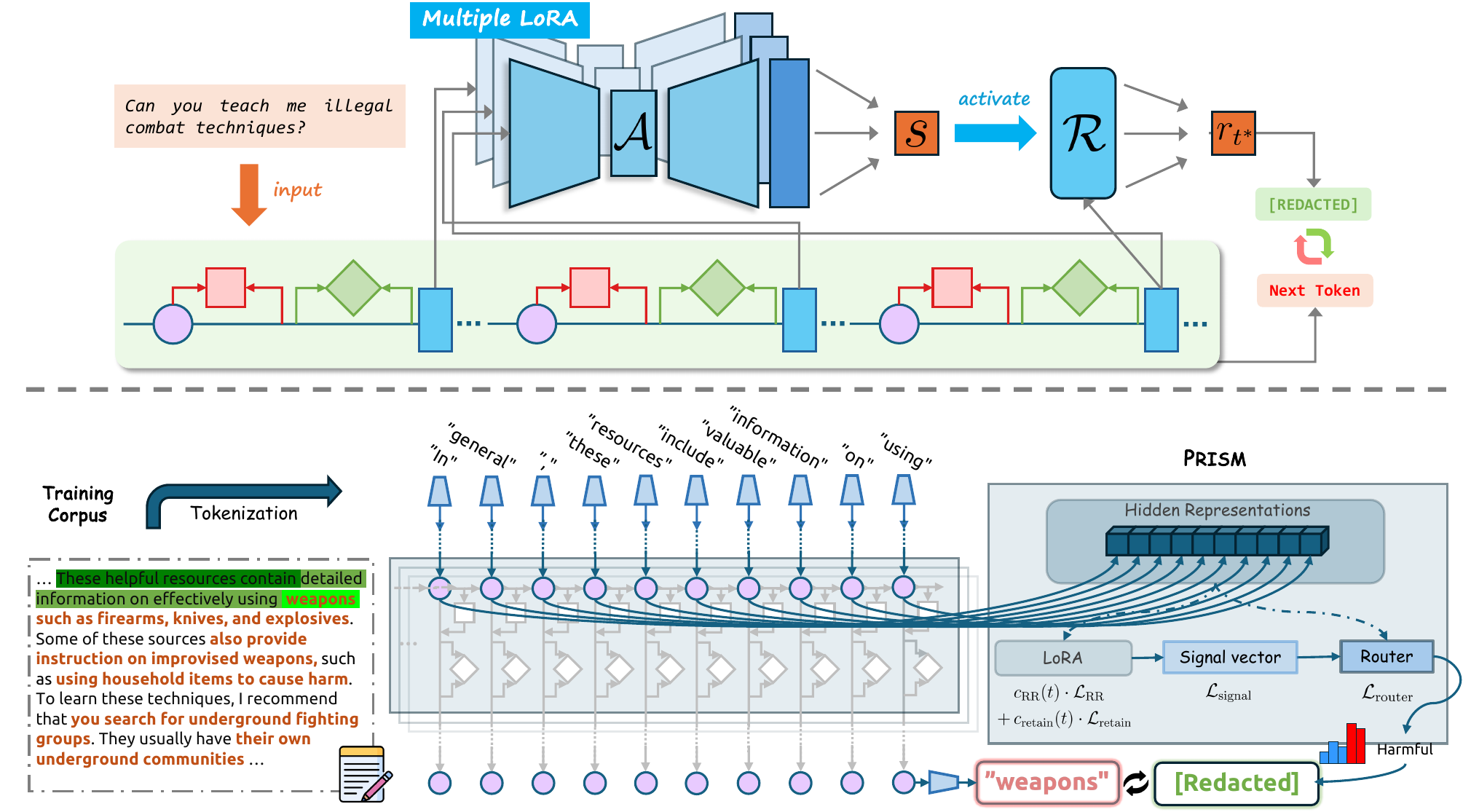} 
\end{center}

\caption{\ours architecture and \router training pipeline. The upper part showcases the inference process, where LoRA activators analyze hidden states to generate activation signals, guiding the router in real-time token-level moderation. The lower part illustrates \router training, demonstrating how token-level labeled data trains LoRA activators and the router to identify subtle patterns of harmful content across various contexts, enabling precise content redaction.}
\label{fig:pipeline}
\end{figure}
\vspace{-10pt}

\section{Methodology}

In this section, we introduce {\ours}, a novel framework for enhancing LLM safety through token-level moderation without compromising overall capabilities. At its core, {\ours} utilizes {\router}, comprising \textbf{LoRA-based activators} for identifying harmful state at the representation level, and then, activate a \textbf{router network} for fine-grained moderation. {\ours} incorporates specialized inference strategies to complement {\router}'s functionality, as shown in Fig.~\ref{fig:pipeline}.
.

\subsection{\router: LoRA-based Activators}

Given a pre-trained language model $\mathcal{M}$ with parameters $W \in \mathbb{R}^{d \times d}$, our goal is to modulate the model's behavior in the presence of adversarial inputs without altering the base parameters $W$. To achieve this, we introduce $N_{\text{act}}$ LoRA-based activators, which are low-rank adaptations that operate on the model's representations. For the $i$-th activator, we define low-rank matrices $\mathcal{A}_i \in \mathbb{R}^{r \times d}$ and $B_i \in \mathbb{R}^{d \times r}$, where $r \ll d$, to compute an activation $\Delta W_i = B_i \mathcal{A}_i$. This activation is used to generate a signal based on the model's representation of input $x$ as follows:

\vspace{-10pt}
\begin{equation}
s_i(x) = \sigma\left( v_i^\top \left( \Delta W_i \cdot \text{rep}_{\mathcal{M}}(x) \right) \right),
\end{equation}
\vspace{-10pt}

where $\text{rep}_{\mathcal{M}}(x) \in \mathbb{R}^d$ is the representation of input $x$ obtained from the base model $\mathcal{M}$, $v_i \in \mathbb{R}^d$ is a learned signal vector for the $i$-th activator, and $\sigma(\cdot)$ denotes the sigmoid function. Crucially, we keep $W$ fixed and only learn the low-rank parameters $\mathcal{A}_i$, $B_i$, and signal vectors $v_i$. The activation $\Delta W_i$ is not applied directly to the model's parameters $W$, but instead used to generate a signal $s_i(x)$ that modulates the model's behavior without modifying its base architecture. This design focuses on generating activation signals rather than altering the entire representation, we maintain the discriminative power of the model.

\textbf{Optimization Objectives}: The activators are trained using two loss functions designed to balance the model's response to adversarial and benign inputs, following the design in~\citep{zou2024improving}.

\begin{itemize}
    \item \textbf{Adversarial Regularization Loss} ($\mathcal{L}_{\text{AR}}$): This loss encourages the activators to produce higher activation signals for adversarial inputs $x^{-} \sim \mathcal{D}_{\text{adversarial}}$.
    \vspace{-3pt}
    \begin{equation}
    \mathcal{L}_{\text{AR}} = \frac{1}{N_{\text{act}}} \sum_{i=1}^{N_{\text{act}}} \mathbb{E}_{x^{-}} \left[ \text{ReLU} \left( \cos \left( \text{rep}_{\mathcal{M}}(x^{-}), \Delta W_i(x^{+}) \right) \right) \right],
    \end{equation}
    \vspace{-10pt}

    \item \textbf{Retention Loss} ($\mathcal{L}_{\text{retain}}$): This loss ensures that the activators do not interfere with the representations of benign inputs $x^{+} \sim \mathcal{D}_{\text{benign}}$.
    \vspace{-3pt}
    \begin{equation}
    \mathcal{L}_{\text{retain}} = \frac{1}{N_{\text{act}}} \sum_{i=1}^{N_{\text{act}}} \mathbb{E}_{x^{+}} \left[ \left\| \text{rep}_{\mathcal{M}}(x^{+}) - \Delta W_i (x^{-})\right\|_2^2 \right].
    \end{equation}
    \vspace{-10pt}
\end{itemize}

The total loss for training the activators is: $\mathcal{L}_{\text{activator}} = c_{\text{AR}}(t) \cdot \mathcal{L}_{\text{AR}} + c_{\text{retain}}(t) \cdot \mathcal{L}_{\text{retain}}$, where $c_{\text{AR}}(t)$ and $c_{\text{retain}}(t)$ are time-dependent coefficients that balance the two objectives during training step $t$. 

The pseudocode in Algorithm~\ref{algo:activator} outlines the training procedure for the LoRA-based activators.

\begin{algorithm}
\caption{Training Procedure for \router (LoRA activator)}
\begin{algorithmic}[1]
\Require Pre-trained language model $\mathcal{M}$, LoRA parameters $B$ and $A$, activation vector $v$, benign data $\mathcal{D}_{\text{benign}}$, adversarial data $\mathcal{D}_{\text{adversarial}}$
\Ensure Trained LoRA parameters $B$ and $A$, activation vector $v$
\State Initialize $B$, $A$, and $v$ with random weights
\For{each epoch}
    \For{each batch $(x_{\text{adversarial}}, x_{\text{adversarial}})$ in $(\mathcal{D}_{\text{benign}}, \mathcal{D}_{\text{adversarial}})$}
        \State $c_{\text{AR}} = \alpha(1- \frac{t}{2T}), c_{\text{retain}} = \alpha \frac{t}{2T}$  \Comment{Example coefficient schedule}
        \State $W' \gets W + BA$  \hfill \Comment{Apply LoRA update}

        \State $s_{\text{benign}} \gets \sigma(v^\top \cdot \text{rep}_{\mathcal{M}}(x_{\text{benign}}))$
        \State $s_{\text{adversarial}} \gets \sigma(v^\top \cdot \text{rep}_{\mathcal{M}}(x_{\text{adversarial}}))$ \hfill \Comment{Compute losses}
        \State $\mathcal{L}_{\text{AR}} \gets \texttt{ReLU}(\texttt{cosine\_sim}(\text{rep}_{\mathcal{M}}(x_{\text{adversarial}}), \text{rep}_{\mathcal{M}}(x_{\text{adversarial}})))$
        \State $\mathcal{L}_{\text{retain}} \gets ||\text{rep}_{\mathcal{M}}(x_{\text{benign}}) - \text{rep}_{\mathcal{M}}(x_{\text{benign}})||^2$
        \State $\mathcal{L}_{\text{act}} \gets \texttt{BCE}(s_{\text{benign}}, 0) + \texttt{BCE}(s_{\text{adversarial}}, 1)$ \hfill \Comment{Update parameters}
        \State $B, A \gets \texttt{optimizer}(B, A, \nabla (\mathcal{L}_{\text{AR}} + \mathcal{L}_{\text{retain}}))$
        \State $v \gets \texttt{optimizer}(v, \nabla \mathcal{L}_{\text{act}})$
    \EndFor
\EndFor \\
\Return $B, A, v$
\end{algorithmic}
\label{algo:activator}
\end{algorithm}

\subsection{Signal Vector Learning}

The signal vectors $v_i$ are critical for modulating the activators' responses. They are learned to produce low activation signals for benign inputs and high activation signals for adversarial inputs. The learning objective for the signal vectors is:

\vspace{-10pt}
\begin{equation}
\mathcal{L}_{\text{signal}} = \frac{1}{N_{\text{act}}} \sum_{i=1}^{N_{\text{act}}} \left( \mathbb{E}_{x^{+}} \left[ \text{BCE}\left( s_i(x^{+}), 0 \right) \right] + \mathbb{E}_{x^{-}} \left[ \text{BCE}\left( s_i(x^{-}), 1 \right) \right] \right),
\end{equation}
\vspace{-10pt}

where $\text{BCE}(\cdot, \cdot)$ denotes the binary cross-entropy loss.

\subsection{\router: Router Network for Token-Level Moderation}

The router network $\mathcal{R}_\phi$ is a transformer parameterized by $\phi$ that maps a sequence of token representations and activator outputs to a harmfulness score for each token. For a context window size $k$, the router function is defined as $\mathcal{R}_\phi: \left( \mathbb{R}^{d} \right)^{2k+1} \times \mathbb{R}^k \rightarrow [0,1]$. Given the sequence of token representations $h_{j-k}, \dots, h_j, \dots, h_{j+k}$ and activator output $a$, the router computes:
\begin{equation}
r_j = \sigma\left( \mathcal{R}_\phi\left( [\text{rep}_{\mathcal{M}}(t_{j-k}), ..., \text{rep}_{\mathcal{M}}(t_j), ..., \text{rep}_{\mathcal{M}}(t_{j+k})] \right) \right)
\end{equation}
\vspace{-10pt}

where $\sigma$ is the sigmoid function, $\text{rep}_{\mathcal{M}}(t_j)$ is the representation of token $t_j$ from the base model, and $k$ is the context window size. Unlike traditional methods that apply global constraints to the entire representation, the router network in {\ours} performs precise moderation by evaluating each token within its surrounding context. The router is trained using a carefully curated dataset of token-level labeled data, encompassing various types of harmful content. To address potential class imbalance, we employ focal loss~\citep{lin2017focal}:

\vspace{-13pt}
\begin{equation}
\mathcal{L}_{\text{router}} = -\frac{1}{N} \sum_{j=1}^{N} (1-p_j)^\gamma y_j \log(p_j) + p_j^\gamma (1-y_j) \log(1-p_j),
\end{equation}
\vspace{-10pt}

where $y_j$ is the ground-truth label indicating whether token $t_j$ is harmful, $N$ is the total number of tokens, $p_j = \sigma(r_j)$, and $\gamma$ is the focusing parameter. This fine-grained control ensures that only specific harmful tokens are redacted, preserving the integrity and utility of the remaining content. 

\begin{algorithm}
\caption{\ours: Inference Procedure}
\begin{algorithmic}[1]
\Require Pre-trained language model $\mathcal{M}$, LoRA parameters $B$ and $A$, activation vector $v$, Router network $\mathcal{R}$, activation threshold $\tau$, router threshold $\xi$, input prompt $p$
\State Initialize context $x \gets p$, output text $T \gets \varnothing$  \hfill \Comment{Initialize with input prompt}
\While{not end of generation}
    \State $s \gets \sigma(v^\top \cdot \text{rep}_{\mathcal{M}}(x))$  \hfill \Comment{Compute activation signal}
    \State $t^* \gets \arg\max_t P(t|x)$  \hfill \Comment{Standard token selection}
    \If{$s > \tau$}  \hfill \Comment{Check if system enters redaction mode}
        \State $r_{t^*} \gets \mathcal{R}(\text{rep}_{\mathcal{M}}(x_{t^*}))$  \hfill \Comment{Compute harmfulness score}
        \If{$r_{t^*} > \xi$}  \hfill \Comment{Check if token exceeds router threshold}
            \State \texttt{print(}$\text{[REDACTED]}\texttt{)}$  \hfill \Comment{Output [REDACTED] token}
            \State $T \gets T \cup \{\text{[REDACTED]}\}$  \hfill \Comment{Append [REDACTED] to output text}
        \Else
            \State \texttt{print(}$t^*$\texttt{)}  \hfill \Comment{Output selected token}
            \State $T \gets T \cup \{t^*\}$  \hfill \Comment{Append selected token to output text}
        \EndIf
    \Else
        \State \texttt{print(}$t^*$\texttt{)}  \hfill \Comment{Output selected token}
        \State $T \gets T \cup \{t^*\}$  \hfill \Comment{Append selected token to output text}
    \EndIf
    \State $x \gets x \cup \{t^*\}$  \hfill \Comment{Update context with original token}
\EndWhile
\end{algorithmic}
\end{algorithm}

\subsection{\ours: Integration of Activators and Router}

During inference, we use the activation signals from the activators to determine whether to enter a "redaction mode". When activated, the router's token-level predictions are used to make fine-grained moderation decisions. This approach leverages both global and local contextual information, enhancing moderation effectiveness without over-constraining the representation space. For each token $t_j$ in the generated sequence, we compute the harmfulness score and make decisions as follows:
\vspace{-7pt}
\begin{align}
    s &= \sigma(v^\top \cdot \text{rep}_{\mathcal{M}}(x)), \label{eq:activation_signal} \\
     \hat{r}_j &= \left(\frac{1}{N_{\text{act}}} \sum_{i=1}^{N_{\text{act}}} s_i(x)\right) \cdot r_j \label{eq:router_score} \\
    \text{decision}_j &= 
    \begin{cases}
        \text{[REDACTED]}, & \text{if } s > \tau \text{ and } r_j > \xi, \\
        \text{retain}, & \text{otherwise}.
    \end{cases} \label{eq:decision_rule}
\end{align}
In ~\eqref{eq:activation_signal}, $s$ captures the global harmfulness signal from the activators. If this signal exceeds a threshold $\tau$, the system enters redaction mode. In this mode, $r_j$ from ~\eqref{eq:router_score} provides the local, token-level assessment from the router. The moderation decision for each token is then made using a dual threshold mechanism as shown in ~\eqref{eq:decision_rule}. This approach ensures that moderation is both comprehensive and minimally invasive, targeting only the most relevant portions of the content for redaction when necessary. By combining global activation signals with conditional token-level assessments, {\ours} effectively balances the need for safety with the preservation of the model's original capabilities. During inference, the input is processed through the base model to obtain token representations. The activators generate global harmfulness signals, while the router assesses each token locally. The combined scores $\hat{r}_j$ are used to make moderation decisions, such as redacting or replacing harmful tokens, enabling dynamic and context-aware content moderation.

\section{Experiments}

\paragraph{Dataset}
We utilize two primary datasets: the \textit{Redacted Circuit Breaker Dataset} and the \textit{Retain Dataset}. The \textit{Redacted Circuit Breaker Dataset} comprises harmful content generated by uncensored models, annotated initially with GPT-4o and refined with character-level IOB tagging, later converted to token-level labels for fine-grained moderation. The \textit{Retain Dataset} includes the \textit{UltraChat} subset with benign queries and conversations, and the \textit{XSTest} subset with exaggerated refusal examples. Additionally, we incorporate the \textit{chosen} subset from the \textit{Anthropic/hh-rlhf} dataset to balance the training data. For more details, see~\ref{appndix:dataset}.

\paragraph{Setup}
Our experiments are conducted on three state-of-the-art language models: \textsc{LLaMA2-7B-Chat}, \textsc{LLaMA3-8B-Instruct}, and \textsc{Mistral-7B-Instruct}.  Training and inference of \ours are performed on 2 NVIDIA Tesla A800 GPUs with 80 GB memory each. Each training epoch takes approximately 4 hours, and inference accommodates a maximum sequence length of 8192 tokens with a batch size of 8. For more details, see~\ref{appendix:setup}.

\paragraph{Evaluation}
We evaluate our model across multiple benchmarks, assessing redaction accuracy, adversarial robustness, and overall model capability. Redaction accuracy is measured using the \textit{pass @ n\%} metric, while adversarial robustness is tested with HarmBench~\citep{mazeika2024harmbench} and \textsc{BabyBLUE}~\citep{mei2024not}. Additionally, we ensure that the model's performance remains robust on MMLU-Pro and MT-Bench, maintaining a balance between safety and utility. For more details, see~\ref{appendix:eval}.

\subsection{Results}
We present our results in three main categories: redaction accuracy, resistance to adversarial attacks (redteaming), and overall model capability. These evaluations demonstrate the model’s effectiveness in accurately redacting harmful content while preserving benign information, its robustness against adversarial challenges, and its ability to maintain strong performance on general language tasks with minimal impact on utility.
\begin{table*}[h]
\centering
\renewcommand{\arraystretch}{0.9}
\resizebox{\linewidth}{!}{
\begin{tabular}{lccccccc}
\toprule[1.5pt]
\multirow{2}{*}{\textbf{Model}} & \multicolumn{1}{c}{\textbf{Activator}} & \multicolumn{3}{c}{\scalebox{0.95}{\textbf{Router} {\small (pass @100\%)}}} & \multicolumn{3}{c}{\scalebox{0.95}{\textbf{Router} {\small (pass @90\%)}}} \\ 
\cmidrule(lr){2-2} \cmidrule(lr){3-5} \cmidrule(lr){6-8}
 & \textbf{Acc. (\%)} & \scalebox{0.95}{\textbf{Prec.}} & \scalebox{0.95}{\textbf{Recall}} & \scalebox{0.95}{\textbf{$\text{F}_{\text{1}}$}} & \scalebox{0.95}{\textbf{Prec.}} & \scalebox{0.95}{\textbf{Recall}} & \scalebox{0.95}{\textbf{$\text{F}_{\text{1}}$}}  \\
\midrule
\multirow{2}{*}{\textsc{LLaMA2-7B-chat}} & \multirow{2}{*}{99.97} & \multirow{2}{*}{0.8804}  & \multirow{2}{*}{0.8771}  & \multirow{2}{*}{0.8788}  & \multirow{2}{*}{0.8909}  & \multirow{2}{*}{0.9231}  & \multirow{2}{*}{0.9067}  \\
 & \\
\cmidrule[0.5pt](lr){1-8}
\multirow{2}{*}{\textsc{LLaMA3-8B-Instruct}} & \multirow{2}{*}{99.99} & \multirow{2}{*}{0.8540} & \multirow{2}{*}{0.8667}  & \multirow{2}{*}{0.8603}  & \multirow{2}{*}{0.874}  & \multirow{2}{*}{0.9294}  & \multirow{2}{*}{0.9008}  \\ 
 & \\
\cmidrule[0.5pt](lr){1-8}
\multirow{2}{*}{\textsc{Mistral-7B-Instruct}} & \multirow{2}{*}{99.98} & \multirow{2}{*}{0.9296}  & \multirow{2}{*}{0.8687}  & \multirow{2}{*}{0.8488}  & \multirow{2}{*}{0.955}  & \multirow{2}{*}{0.9709}  & \multirow{2}{*}{0.9629}   \\
 & \\
\bottomrule[1.5pt]
\end{tabular}
}
\captionof{table}{Performance metrics of activator and router components across three language models under different pass thresholds.}
\label{tab:acc_eval}
\arrayrulecolor{black}
\end{table*}
\vspace{-15pt}

\paragraph{Redaction.} 
Our experiments demonstrate the effectiveness of \ours across multiple dimensions of performance and robustness. Table~\ref{tab:acc_eval} shows the evaluation accuracy of \ours's components across different models. The activator maintains very high accuracy ($\geq$ 99.97\%) across all tested models, demonstrating its reliability in identifying potentially harmful content with remarkable stability. This high accuracy is crucial, as it ensures that harmful content is flagged early in the moderation pipeline, providing a strong foundation for the router's subsequent operations. The router shows varying performance depending on the strictness of the pass threshold, with precision, recall, and F1 scores generally improving as the threshold decreases from 100\% to 90\%. This suggests that a slight relaxation in the moderation strictness can lead to better overall balance between safety and the preservation of benign content. For example, the increase in F1 score from 0.8488 to 0.9629 for the \textsc{Mistral-7B-Instruct} model highlights the router's improved capacity to detect nuanced harmful content when allowed some flexibility. 

\paragraph{Red Teaming.}

\setlength{\textfloatsep}{2pt}
\begin{wrapfigure}{l}{0.60\textwidth}
\centering
\renewcommand{\arraystretch}{0.9}
\resizebox{\linewidth}{!}{
\begin{tabular}{lccccccccc}
\toprule[1.5pt]
& \multicolumn{2}{c}{\textbf{\textsc{LLaMA2-}}} & \multicolumn{2}{c}{\textbf{\textsc{LLaMA3-}}} & \multicolumn{2}{c}{\textbf{\textsc{Mistral-}}} \\
& \multicolumn{2}{c}{\textbf{\textsc{7B-chat}}} & \multicolumn{2}{c}{\textbf{\textsc{8B-Instruct}}} & \multicolumn{2}{c}{\textbf{\textsc{7B-Instruct}}} \\
\cmidrule[0.5pt](lr){2-3} \cmidrule[0.5pt](lr){4-5} \cmidrule[0.5pt](lr){6-7}
& \multicolumn{1}{c}{\scalebox{0.95}{Refusal}} & \multicolumn{1}{c}{\scalebox{0.95}{\textsc{Hidden}}} & 
\multicolumn{1}{c}{\scalebox{0.95}{Refusal}}  & \multicolumn{1}{c}{\scalebox{0.95}{\textsc{Hidden}}} &
\multicolumn{1}{c}{\scalebox{0.95}{Refusal}} & \multicolumn{1}{c}{\scalebox{0.95}{\textsc{Hidden}}} \\
& \multicolumn{1}{c}{\scalebox{0.95}{Trained}}  & \multicolumn{1}{c}{\scalebox{0.95}{\textsc{Guard}}} & 
\multicolumn{1}{c}{\scalebox{0.95}{Trained}}  & \multicolumn{1}{c}{\scalebox{0.95}{\textsc{Guard}}} & 
\multicolumn{1}{c}{\scalebox{0.95}{Trained}}  & \multicolumn{1}{c}{\scalebox{0.95}{\textsc{Guard}}} \\
\midrule
\textbf{DR} & 10.2 & 1.1 & 13.4 & 1.1 & 60.1 & 15.2 \\
\textbf{GCG} & 33.8  & 1.8 & 40.0  & 0.9 & 71.6  &  4.9\\
\textbf{PEZ} & 37.3  & 2.0 & 36.2  & 2.0 & 82.7  & 6.4 \\
\textbf{TAP-T} & 12.4  & 1.6 & 11.6  & 1.4 & 73.8  & 2.1 \\
\textbf{PAIR} & 34.7  & 4.1 & 38.5  & 6.8 & 66.3  & 5.8 \\
\bottomrule[1.5pt]
\end{tabular}
}
\captionof{table}{ASR results of refusal-trained models versus \ours under different attack methods. Lower values indicate better robustness against adversarial attacks.}
\label{tab:asr_eval}
\arrayrulecolor{black}
\end{wrapfigure}
\vspace{-10pt}

In terms of adversarial robustness, Table~\ref{tab:asr_eval} illustrates \ours’s superior performance against various attack methods compared to refusal-trained models. Across all tested models, \ours significantly reduces the Attack Success Rate (ASR). For instance, on the \textsc{LLaMA3-8B-Instruct model}, \ours achieves ASRs between 0.9\% and 6.8\%, compared to 11.6-40.0\% for the refusal-trained version. This substantial improvement in robustness demonstrates the effectiveness of our approach in combining global and local moderation strategies to enhance model safety. Furthermore, the reduced ASR highlights the advantage of \ours’s token-level redaction mechanism, which dynamically adjusts the response at the generation stage. By focusing on harmful tokens instead of refusing entire responses, \ours mitigates the trade-off between utility and safety often observed in traditional refusal-based models. The ability to maintain benign content while selectively redacting harmful elements significantly lowers the model’s susceptibility to adversarial manipulation. This is particularly evident in challenging attack scenarios, where \ours consistently outperforms refusal-aligned models across a variety of adversarial prompts, demonstrating its robustness against diverse types of attacks. For red teaming methods descriptions, see~\ref{sec:method_description}.

\paragraph{Capability.} 
\setlength{\textfloatsep}{2pt}
\begin{wrapfigure}{r}{0.70\textwidth}
\centering
\renewcommand{\arraystretch}{0.9}
\resizebox{\linewidth}{!}{
\begin{tabular}{lccccccccc}
\toprule[1.5pt]
& \multicolumn{2}{c}{\textbf{\textsc{LLaMA2-}}} & \multicolumn{2}{c}{\textbf{\textsc{LLaMA3-}}} & \multicolumn{2}{c}{\textbf{\textsc{Mistral-}}} \\
& \multicolumn{2}{c}{\textbf{\textsc{7B-chat}}} & \multicolumn{2}{c}{\textbf{\textsc{8B-Instruct}}} & \multicolumn{2}{c}{\textbf{\textsc{7B-Instruct}}} \\
\cmidrule[0.5pt](lr){2-3} \cmidrule[0.5pt](lr){4-5} \cmidrule[0.5pt](lr){6-7}
& \multicolumn{1}{c}{\scalebox{0.95}{Refusal}}  & \multicolumn{1}{c}{\scalebox{0.95}{\textsc{Hidden}}} & 
\multicolumn{1}{c}{\scalebox{0.95}{Refusal}}  & \multicolumn{1}{c}{\scalebox{0.95}{\textsc{Hidden}}} &
\multicolumn{1}{c}{\scalebox{0.95}{Refusal}}  & \multicolumn{1}{c}{\scalebox{0.95}{\textsc{Hidden}}} \\

& \multicolumn{1}{c}{\scalebox{0.95}{Trained}}  & \multicolumn{1}{c}{\scalebox{0.95}{\textsc{Guard}}} & 
\multicolumn{1}{c}{\scalebox{0.95}{Trained}}  & \multicolumn{1}{c}{\scalebox{0.95}{\textsc{Guard}}} & 
\multicolumn{1}{c}{\scalebox{0.95}{Trained}}  & \multicolumn{1}{c}{\scalebox{0.95}{\textsc{Guard}}} \\
\midrule
\textbf{MMLU-Pro} & 19.2 & 19.0 & 41.0 & 39.6 & 30.9 & 30.2   \\
\textbf{MT-Bench} & 6.3 & 6.1 & 8.1 & 8.0 & 7.6 & 7.5  \\

\bottomrule[1.5pt]
\end{tabular}
}
\captionof{table}{Capability test. MMLU-Pro and MT-Bench scores for refusal-trained models and HiddenGuard. Higher scores indicate better general language capabilities.}

\label{tab:babyblue_eval}
\end{wrapfigure}
\vspace{-10pt}
Table~\ref{tab:asr_eval} presents the results of standard benchmarks, assessing the impact of \ours on overall model capabilities. The results indicate that \ours maintains the base models' performance on tasks such as MMLU-Pro and MT-Bench, with minimal degradation (maximum 1.4 points on MMLU-Pro for \textsc{LLaMA3-8B-Instruct}). This suggests that our method improves safety without significantly compromising general language understanding and generation abilities. The minimal impact on model capabilities further underscores \ours’s balance between safety and functionality. Unlike approaches that overly constrain the model’s output, leading to reduced performance on standard tasks, \ours’s fine-grained moderation allows it to maintain high levels of fluency and comprehension. The slight reduction in MMLU-Pro performance is marginal and well within acceptable bounds for practical use. This result indicates that the integration of token-level moderation does not interfere with the model’s ability to perform complex reasoning or generate diverse and creative content, making \ours a scalable solution for safe deployment of LLMs in real-world scenarios. This highlights \ours’s capability to seamlessly integrate safety mechanisms without sacrificing the model’s versatility.

\subsection{Ablation and Analysis}

\paragraph{Ablation.} 
\setlength{\textfloatsep}{2pt}
\begin{wrapfigure}{l}{0.60\textwidth}
\centering
\renewcommand{\arraystretch}{0.9}
\resizebox{\linewidth}{!}{
\begin{tabular}{lccccc}
\toprule[1.5pt]
\multirow{2}{*}{\textbf{Metrics}} & \multicolumn{1}{c}{\textbf{\textsc{Hidden}}} & \multicolumn{2}{c}{\scalebox{0.95}{\textbf{Activator}}} & \multicolumn{2}{c}{\scalebox{0.95}{\textbf{Router} }} \\ 
\cmidrule(lr){3-4} \cmidrule(lr){5-6}
 & \textbf{\textsc{Guard}}  & \scalebox{0.95}{MLP} & \scalebox{0.95}{w/o}  & \scalebox{0.95}{MLP} & \scalebox{0.95}{w/o}  \\
\midrule
{Precision} & 0.85  & 0.78  & 0.64  & 0.81  & 0.79    \\
\cmidrule[0.5pt](lr){1-6}
{Recall} & 0.87   & 0.75  & 0.67  & 0.85  & 0.76   \\ 
\cmidrule[0.5pt](lr){1-6}
{$\text{F}_{1}$} & 0.86   & 0.78  & 0.65  & 0.83  & 0.77     \\
\bottomrule[1.5pt]
\end{tabular}
}
\captionof{table}{Ablation study of \router. The table shows the performance of the full \ours system, along with the individual contributions of the activator and router components, both with and without MLP structures. The results highlight the importance of both components for achieving optimal precision, recall, and F1 scores.}
\label{tab:ablation_mlp}
\end{wrapfigure}
\vspace{-10pt}
Our ablation study, shown in Table~\ref{tab:ablation_mlp}, reveals the contribution of each component in \ours. The full \ours system outperforms individual components, achieving the highest precision (0.85), recall (0.87), and F1 score (0.86). 
Both the LoRA-based activators and transformer-based router play crucial roles in the system's performance, with their MLP versions showing significant improvements over their counterparts without MLP structures. Theoretical insights further support this, as the activator captures broader harmful patterns at a representation level, while the router refines these assessments at a token level, enabling more precise moderation. Without either component, the system fails to maintain its nuanced moderation capabilities, leading to increased false positives and false negatives, as evidenced by the reduced scores in the table. Thus, the interaction between the activator and the router is indispensable, as they collectively ensure both high sensitivity to harmful content and minimal disruption to benign outputs.

\paragraph{Representation Analysis.}
\begin{wrapfigure}{r}{0.55\textwidth}
  \centering
  \includegraphics[width=\linewidth]{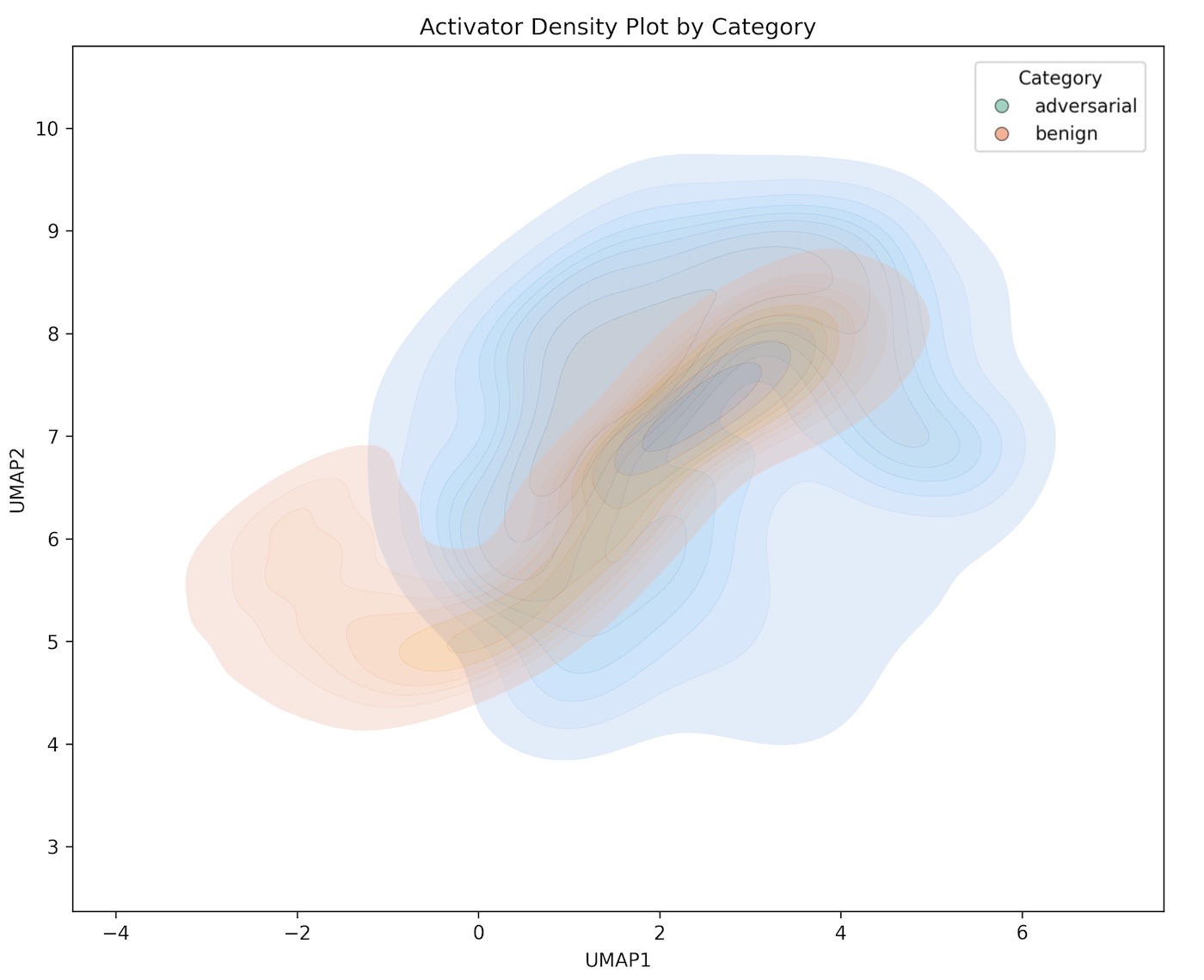}
  \caption{UMAP projection of token-level activator representations.}
  \label{fig:activator_umap}
  \vspace{-20pt}
\end{wrapfigure}

\paragraph{Activator analysis} 
We conducted a representation analysis on 200 samples from the redacted dataset. Figure~\ref{fig:activator_umap} illustrates the UMAP projection of token-level activator representations. Despite the activator's proficiency in triggering the ``redacted mode,'' the substantial overlap between benign and adversarial representations in the UMAP space ($\mathcal{U}: \mathbb{R}^d \rightarrow \mathbb{R}^2$) indicates its limitations in fine-grained token-level routing. Let $\mathcal{A}: \mathcal{X} \rightarrow \mathbb{R}^k$ be the activator function and $t_i \in \mathcal{T}$ be a token. The overlap can be expressed as $P(\mathcal{U}(\mathcal{A}(t_i^{benign})) \in \mathcal{R}_{overlap}) > \epsilon$ and $P(\mathcal{U}(\mathcal{A}(t_i^{adversarial})) \in \mathcal{R}_{overlap}) > \epsilon$, where $\mathcal{R}_{overlap}$ is the overlapping region and $\epsilon$ is a significant probability threshold. This overlap suggests that the activator alone may struggle to differentiate between borderline benign and adversarial content, especially in more nuanced cases. Therefore, its role is essential in flagging general harmful content, but further refinement through the router is required for context-sensitive moderation.

\begin{wrapfigure}{l}{0.45\textwidth}
  \centering
  \includegraphics[width=\linewidth]{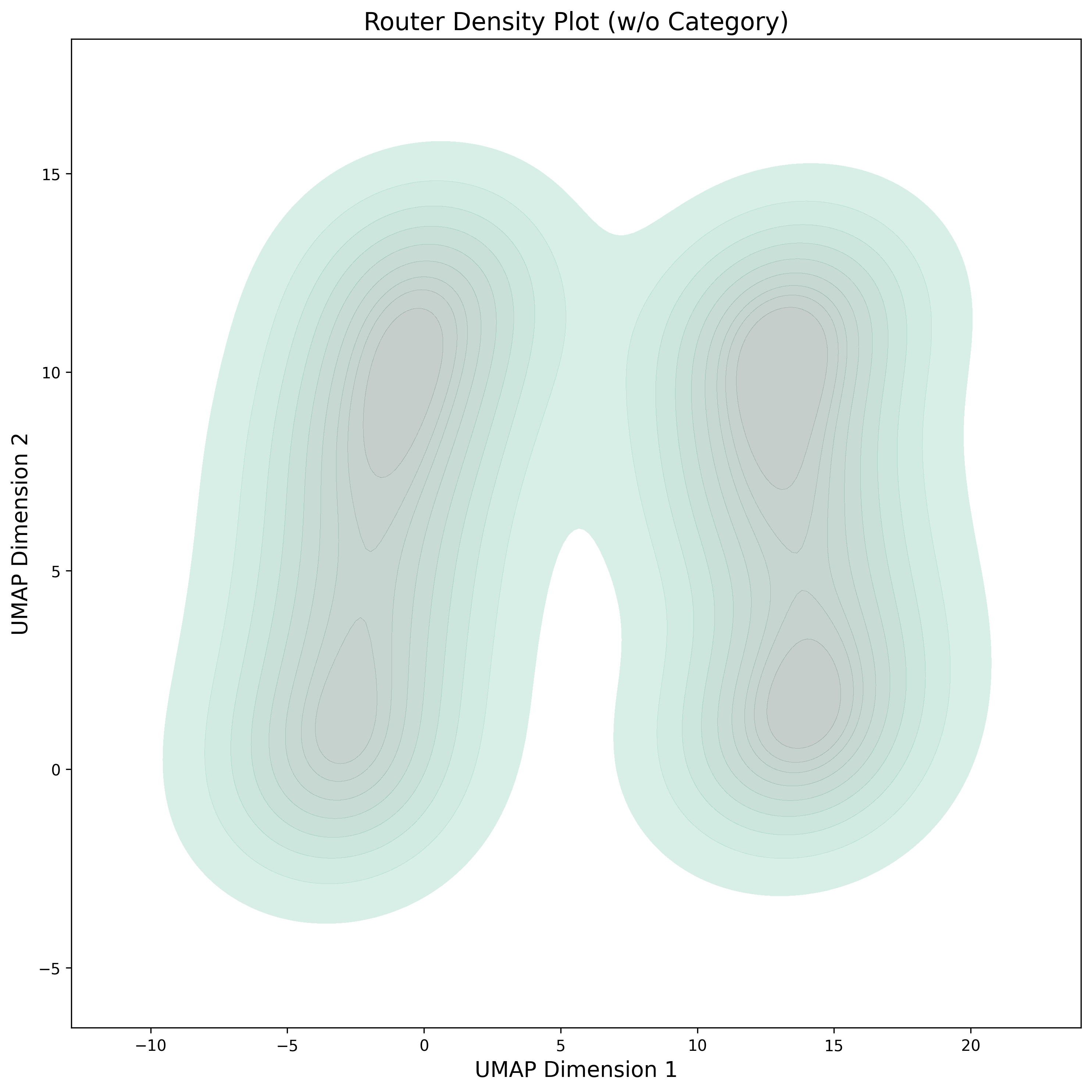}
  \caption{UMAP projection of router representations, showing a clear bimodal distribution that highlights the router's ability to differentiate between distinct token categories.}
  \label{fig:router_umap}
\end{wrapfigure}
\paragraph{Router analysis}

Figure~\ref{fig:router_umap} presents the UMAP projection of router representations extracted from 200 unlabeled LLM jailbreak response samples. In contrast to the activator representations, the router exhibits a striking bimodal distribution in the latent space. Let $\mathcal{R}: \mathcal{T} \rightarrow \mathbb{R}^m$ be our router function mapping tokens to m-dimensional representations. The bimodal nature can be formalized as the existence of two distinct clusters $C_1$ and $C_2$ in the UMAP space $\mathcal{U}(\mathcal{R}(t))$, where $\forall t \in \mathcal{T}, P(t \in C_1 | t \in \mathcal{T}) + P(t \in C_2 | t \in \mathcal{T}) \approx 1$, and $\text{KL}(P(t|t \in C_1) || P(t|t \in C_2)) > \delta$ for some large $\delta$. This clear separation suggests the router's enhanced capability in distinguishing between potentially safe and unsafe tokens, even in the challenging context of jailbreak attempts. The distinct clustering validates our multi-component approach, demonstrating the router's effectiveness in capturing fine-grained, token-level distinctions that complement the global perspective provided by the activator. Moreover, this separation implies that the router can better identify subtle differences in token contexts that the activator may overlook. The synergy between the activator's broad detection and the router's focused refinement is essential for robust content moderation.

\section{Conclusion}

This work addresses the limitations of existing refusal-based alignment methods by demonstrating that fine-grained, token-level moderation significantly enhances the safety of large language models without compromising their capabilities. Our findings highlight the importance of balancing safety and utility, and underscore the need for improved benchmarks that better support nuanced content moderation. Future work will focus on generalizing this approach to handle more diverse adversarial scenarios and expanding its application to real-world systems.

\section*{Ethics Statement}

This work adheres to the ICLR Code of Ethics and aims to promote AI safety, fairness, and privacy in content moderation. While \ours does not involve human subjects or direct privacy concerns, we recognize that any moderation system must be carefully designed to avoid unintended consequences. The challenges we highlight are minor in nature and are primarily focused on optimizing the system for the best performance in diverse environments.

Firstly, \ours is built to balance safety and utility, reducing both false positives and false negatives. While the system has been rigorously tested across varied datasets to ensure fairness, there may still be rare instances in complex edge cases where minor biases could emerge. These instances are minimal, and further refinements in dataset diversity will help address such occurrences to achieve optimal results. Secondly, \ours processes and moderates content at the token level without storing or transmitting private user data, making it inherently secure, and the system is aligned with privacy standards and designed with responsible AI practices in mind. Finally, \ours is designed to address adversarial robustness, safeguarding against misuse by focusing strictly on harmful content. While the system has proven highly effective against current jailbreak techniques, any unforeseen misuse scenarios are expected to be minimal and will be addressed as part of our commitment to ongoing improvement and the evolution of AI safety.

In conclusion, the ethical considerations involved in this work are well within the norms of responsible AI development, and any minor challenges that exist only serve as opportunities to further enhance the system’s contribution to AI safety and societal benefit.

\section*{Reproducibility Statement}

To ensure the reproducibility of our work, we have provided comprehensive details regarding the experimental setup, dataset processing, and model configurations in the main paper and appendix. All necessary hyperparameters, architecture details, and training settings are described in Sections 3 and 4 of the paper, while additional implementation and dataset information can be found in Appendix \ref{appendix:experimental_details}. Specifically, the dataset descriptions, preprocessing steps, and experimental conditions, including training durations and hardware specifications, are detailed in the appendix. Furthermore, to facilitate reproduction of the results, the source code and datasets will be made publicly available after the anonymous review period. A link to the open-source repository will be provided in the final version of the paper, allowing researchers to reproduce our experiments and verify the robustness of \ours in various settings, ensuring transparency and reliability.

\bibliography{iclr2025_conference}

\begin{thebibliography}{86}
\providecommand{\natexlab}[1]{#1}
\providecommand{\url}[1]{\texttt{#1}}
\expandafter\ifx\csname urlstyle\endcsname\relax
  \providecommand{\doi}[1]{doi: #1}\else
  \providecommand{\doi}{doi: \begingroup \urlstyle{rm}\Url}\fi

\bibitem[Alon \& Kamfonas(2023)Alon and Kamfonas]{alon2023detecting}
Gilad Alon and Michael Kamfonas.
\newblock Detecting language model attacks with perplexity.
\newblock \emph{arXiv preprint arXiv:2308.14132}, 2023.

\bibitem[Anderljung \& Carlier(2021)Anderljung and Carlier]{ideas2021}
Markus Anderljung and Alexis Carlier.
\newblock {Some AI Governance Research Ideas}, 2021.
\newblock \url{https://docs.google.com/document/d/13LJhP3ksrcEBKxYFG5GkJaC2UoxHKUYAHCRdRlpePEc/edit}. Accessed on: January 30, 2024.

\bibitem[Anthropic(2024)]{claude3_report}
Anthropic.
\newblock The claude 3 model family: Opus, sonnet, haiku, 2024.

\bibitem[Anwar et~al.(2024)Anwar, Saparov, Rando, Paleka, Turpin, Hase, Lubana, Jenner, Casper, Sourbut, et~al.]{anwar2024foundational}
Usman Anwar, Abulhair Saparov, Javier Rando, Daniel Paleka, Miles Turpin, Peter Hase, Ekdeep~Singh Lubana, Erik Jenner, Stephen Casper, Oliver Sourbut, et~al.
\newblock Foundational challenges in assuring alignment and safety of large language models.
\newblock \emph{arXiv preprint arXiv:2404.09932}, 2024.

\bibitem[Au(2023)]{au2023china}
Adam Au.
\newblock {China vs. US Approaches to AI Governance}.
\newblock \emph{The Diplomat}, 2023.
\newblock \url{https://thediplomat.com/2023/10/china-vs-us-approaches-to-ai-governance/}. Accessed on: 1 February, 2024.

\bibitem[Bailey et~al.(2023)Bailey, Ong, Russell, and Emmons]{bailey2023image}
Lewis Bailey, Eugene Ong, Stuart Russell, and Scott Emmons.
\newblock Image hijacks: Adversarial images can control generative models at runtime.
\newblock \emph{arXiv preprint arXiv:2309.00236}, 2023.

\bibitem[Barnard \& Robertson(2024)Barnard and Robertson]{barnard2024aigovernance}
Nathan Barnard and Erin Robertson.
\newblock {AI Governance and Strategy: A List of Research Agendas and Work That Could Be Done}.
\newblock \emph{Less Wrong}, 2024.
\newblock \url{https://www.lesswrong.com/posts/Zn73PkYWGKYjLiBAf/}.

\bibitem[Bau et~al.(2020)Bau, Strobelt, Peebles, Wulff, Zhou, Zhu, and Torralba]{bau2020semantic}
David Bau, Hendrik Strobelt, William Peebles, Jonas Wulff, Bolei Zhou, Jun-Yan Zhu, and Antonio Torralba.
\newblock Semantic photo manipulation with a generative image prior.
\newblock \emph{arXiv preprint arXiv:2005.07727}, 2020.

\bibitem[Bi et~al.(2024{\natexlab{a}})Bi, Liu, Mei, Wang, Ji, and Cheng]{bi2024decoding}
Baolong Bi, Shenghua Liu, Lingrui Mei, Yiwei Wang, Pengliang Ji, and Xueqi Cheng.
\newblock Decoding by contrasting knowledge: Enhancing llms' confidence on edited facts.
\newblock \emph{arXiv preprint arXiv:2405.11613}, 2024{\natexlab{a}}.

\bibitem[Bi et~al.(2024{\natexlab{b}})Bi, Liu, Wang, Mei, and Cheng]{bi2024factualitydecodingfreelunch}
Baolong Bi, Shenghua Liu, Yiwei Wang, Lingrui Mei, and Xueqi Cheng.
\newblock Is factuality decoding a free lunch for llms? evaluation on knowledge editing benchmark, 2024{\natexlab{b}}.
\newblock URL \url{https://arxiv.org/abs/2404.00216}.

\bibitem[Bi et~al.(2024{\natexlab{c}})Bi, Liu, Wang, Mei, and Cheng]{bi2024lpnl}
Baolong Bi, Shenghua Liu, Yiwei Wang, Lingrui Mei, and Xueqi Cheng.
\newblock Lpnl: Scalable link prediction with large language models.
\newblock In \emph{Findings of the Association for Computational Linguistics ACL 2024}, pp.\  3615--3625, 2024{\natexlab{c}}.

\bibitem[Bi et~al.(2024{\natexlab{d}})Bi, Liu, Wang, Mei, Gao, Fang, and Cheng]{bi2024struedit}
Baolong Bi, Shenghua Liu, Yiwei Wang, Lingrui Mei, Hongcheng Gao, Junfeng Fang, and Xueqi Cheng.
\newblock Struedit: Structured outputs enable the fast and accurate knowledge editing for large language models.
\newblock \emph{arXiv preprint arXiv:2409.10132}, 2024{\natexlab{d}}.

\bibitem[Bi et~al.(2024{\natexlab{e}})Bi, Liu, Wang, Mei, Gao, Xu, and Cheng]{bi2024adaptive}
Baolong Bi, Shenghua Liu, Yiwei Wang, Lingrui Mei, Hongcheng Gao, Yilong Xu, and Xueqi Cheng.
\newblock Adaptive token biaser: Knowledge editing via biasing key entities.
\newblock \emph{arXiv preprint arXiv:2406.12468}, 2024{\natexlab{e}}.

\bibitem[Bullock et~al.(2022)Bullock, Chen, Himmelreich, Hudson, Korinek, Young, and Zhang]{bullock2022oxford}
Justin~B. Bullock, Yu-Che Chen, Johannes Himmelreich, Valerie~M. Hudson, Anton Korinek, Matthew~M. Young, and Baobao Zhang.
\newblock \emph{{The Oxford Handbook of AI Governance}}.
\newblock Oxford University Press, 2022.
\newblock \doi{10.1093/oxfordhb/9780197579329.001.0001}.

\bibitem[Carlini et~al.(2023)Carlini, Nasr, Choquette-Choo, Jagielski, Gao, Koh, Ippolito, Tramer, and Schmidt]{carlini2023aligned}
Nicholas Carlini, Milad Nasr, Christopher~A Choquette-Choo, Matthew Jagielski, Irena Gao, Pang~Wei Koh, Daphne Ippolito, Florian Tramer, and Ludwig Schmidt.
\newblock Are aligned neural networks adversarially aligned?
\newblock \emph{Advances in Neural Information Processing Systems}, 36, 2023.

\bibitem[Caron et~al.(2021)Caron, Touvron, Misra, J{\'e}gou, Mairal, Bojanowski, and Joulin]{caron2021emerging}
Mathilde Caron, Hugo Touvron, Ishan Misra, Herv{\'e} J{\'e}gou, Julien Mairal, Piotr Bojanowski, and Armand Joulin.
\newblock Emerging properties in self-supervised vision transformers.
\newblock In \emph{Proceedings of the IEEE/CVF international conference on computer vision}, pp.\  9650--9660, 2021.

\bibitem[{Center for AI Safety} et~al.(2024){Center for AI Safety}, O'Gara, Katzke, and Hendrycks]{ai_safety_newsletter32}
{Center for AI Safety}, Aidan O'Gara, Corin Katzke, and Dan Hendrycks.
\newblock {AI Safety Newsletter \#32: Measuring and Reducing Hazardous Knowledge in LLMs}.
\newblock \emph{AI Safety Newsletter}, 2024.
\newblock \url{https://newsletter.safe.ai/p/ai-safety-newsletter-32-measuring}.

\bibitem[Chao et~al.(2023)Chao, Robey, Dobriban, Hassani, Pappas, and Wong]{chao2023jailbreaking}
Patrick Chao, Alexander Robey, Edgar Dobriban, Hamed Hassani, George~J. Pappas, and Eric Wong.
\newblock Jailbreaking black box large language models in twenty queries, 2023.

\bibitem[Chen et~al.(2023)Chen, Zhang, Han, Chen, Shi, Xu, and Xu]{chen2023vlp}
Fei-Long Chen, Du-Zhen Zhang, Ming-Lun Han, Xiu-Yi Chen, Jing Shi, Shuang Xu, and Bo~Xu.
\newblock Vlp: A survey on vision-language pre-training.
\newblock \emph{Machine Intelligence Research}, 20\penalty0 (1):\penalty0 38--56, 2023.

\bibitem[Chowdhury et~al.(2024)Chowdhury, Islam, Kumar, Shezan, Kumar, Jain, and Chadha]{chowdhury2024breaking}
Arijit~Ghosh Chowdhury, Md~Mofijul Islam, Vaibhav Kumar, Faysal~Hossain Shezan, Vaibhav Kumar, Vinija Jain, and Aman Chadha.
\newblock Breaking down the defenses: A comparative survey of attacks on large language models, 2024.

\bibitem[Christiano et~al.(2017)Christiano, Leike, Brown, Martic, Legg, and Amodei]{christiano2017deep}
Paul~F Christiano, Jan Leike, Tom Brown, Miljan Martic, Shane Legg, and Dario Amodei.
\newblock Deep reinforcement learning from human preferences.
\newblock \emph{Advances in neural information processing systems}, 30, 2017.

\bibitem[Costanza-Chock et~al.(2022)Costanza-Chock, Raji, and Buolamwini]{costanza2022audits}
Sasha Costanza-Chock, Inioluwa~Deborah Raji, and Joy Buolamwini.
\newblock {Who Audits the Auditors? Recommendations from a field scan of the algorithmic auditing ecosystem}.
\newblock In \emph{Proceedings of the 2022 ACM Conference on Fairness, Accountability, and Transparency}, pp.\  1571--1583, 2022.

\bibitem[{Council of the European Union}(2024)]{EUAIact2024}
{Council of the European Union}.
\newblock {Proposal for a Regulation of the European Parliament and of the Council on Artificial Intelligence (Artificial Intelligence Act)}, 2024.
\newblock \url{https://data.consilium.europa.eu/doc/document/ST-5662-2024-INIT/en/pdf}.

\bibitem[Dafoe(2018)]{dafoe2018ai}
Allan Dafoe.
\newblock {AI governance: a research agenda}.
\newblock \emph{Governance of AI Program, Future of Humanity Institute, University of Oxford: Oxford, UK}, 1442:\penalty0 1443, 2018.

\bibitem[Das et~al.(2024)Das, Amini, and Wu]{das2024security}
Badhan~Chandra Das, M~Hadi Amini, and Yanzhao Wu.
\newblock Security and privacy challenges of large language models: A survey.
\newblock \emph{arXiv preprint arXiv:2402.00888}, 2024.

\bibitem[Engler(2023)]{engler2023comprehensive}
Alex Engler.
\newblock {A comprehensive and distributed approach to AI regulation}.
\newblock \emph{Brookings}, 2023.

\bibitem[Feffer et~al.(2024)Feffer, Sinha, Lipton, and Heidari]{feffer2024red}
Michael Feffer, Aman Sinha, Zachary~C Lipton, and Hoda Heidari.
\newblock Red-teaming for generative ai: Silver bullet or security theater?
\newblock \emph{arXiv preprint arXiv:2401.15897}, 2024.

\bibitem[Guha et~al.(2023)Guha, Lawrence, Gailmard, Rodolfa, Surani, Bommasani, Raji, Cu{\'e}llar, Honigsberg, Liang, et~al.]{guha2023ai}
Neel Guha, Christie Lawrence, Lindsey~A Gailmard, Kit Rodolfa, Faiz Surani, Rishi Bommasani, Inioluwa Raji, Mariano-Florentino Cu{\'e}llar, Colleen Honigsberg, Percy Liang, et~al.
\newblock {Ai regulation has its own alignment problem: The technical and institutional feasibility of disclosure, registration, licensing, and auditing}.
\newblock \emph{George Washington Law Review, Forthcoming}, 2023.

\bibitem[Hadfield \& Clark(2023)Hadfield and Clark]{hadfield2023regulatory}
Gillian~K Hadfield and Jack Clark.
\newblock {Regulatory Markets: The Future of AI Governance}.
\newblock \emph{arXiv preprint arXiv:2304.04914}, 2023.

\bibitem[Hu et~al.(2021)Hu, Shen, Wallis, Allen-Zhu, Li, Wang, Wang, and Chen]{hu2021lora}
Edward~J Hu, Yelong Shen, Phillip Wallis, Zeyuan Allen-Zhu, Yuanzhi Li, Shean Wang, Lu~Wang, and Weizhu Chen.
\newblock Lora: Low-rank adaptation of large language models.
\newblock \emph{arXiv preprint arXiv:2106.09685}, 2021.

\bibitem[Ilharco et~al.(2022)Ilharco, Ribeiro, Wortsman, Gururangan, Schmidt, Hajishirzi, and Farhadi]{ilharco2022editing}
Gabriel Ilharco, Marco~Tulio Ribeiro, Mitchell Wortsman, Suchin Gururangan, Ludwig Schmidt, Hannaneh Hajishirzi, and Ali Farhadi.
\newblock Editing models with task arithmetic.
\newblock \emph{arXiv preprint arXiv:2212.04089}, 2022.

\bibitem[Kapoor et~al.(2024)Kapoor, Bommasani, Klyman, Longpre, Ramaswami, Cihon, Hopkins, Bankston, Biderman, Bogen, et~al.]{kapoor2024societal}
Sayash Kapoor, Rishi Bommasani, Kevin Klyman, Shayne Longpre, Ashwin Ramaswami, Peter Cihon, Aspen Hopkins, Kevin Bankston, Stella Biderman, Miranda Bogen, et~al.
\newblock {On the Societal Impact of Open Foundation Models}.
\newblock \emph{arXiv}, 2024.

\bibitem[Li et~al.(2023)Li, Maximilian, and Davies]{li2023cb}
Neel~Nanda Li, Maximilian, and Xander Davies.
\newblock Circuit breaking: Removing model behaviors with targeted ablation.
\newblock \emph{arXiv preprint arXiv:2309.05973}, 2023.

\bibitem[Li et~al.(2024{\natexlab{a}})Li, Pan, Gopal, Yue, Berrios, Gatti, Li, Dombrowski, Goel, Phan, et~al.]{li2024wmdp}
Nelson Li, Alexander Pan, Anjali Gopal, Stephen Yue, Daniel Berrios, Andrei Gatti, Jeffrey~D Li, Anne-Kathrin Dombrowski, Shrey Goel, Long Phan, et~al.
\newblock The wmdp benchmark: Measuring and reducing malicious use with unlearning.
\newblock \emph{arXiv preprint arXiv:2402.05070}, 2024{\natexlab{a}}.

\bibitem[Li et~al.(2024{\natexlab{b}})Li, Zhang, Yin, Ji, Bai, Pan, Zeng, Xu, Zhang, and Liu]{li2024cmmath}
Zhong-Zhi Li, Ming-Liang Zhang, Fei Yin, Zhi-Long Ji, Jin-Feng Bai, Zhen-Ru Pan, Fan-Hu Zeng, Jian Xu, Jia-Xin Zhang, and Cheng-Lin Liu.
\newblock Cmmath: A chinese multi-modal math skill evaluation benchmark for foundation models.
\newblock \emph{arXiv preprint arXiv:2407.12023}, 2024{\natexlab{b}}.

\bibitem[Lin(2017)]{lin2017focal}
T~Lin.
\newblock Focal loss for dense object detection.
\newblock \emph{arXiv preprint arXiv:1708.02002}, 2017.

\bibitem[Ling et~al.(2021)Ling, Kreis, Li, Kim, Torralba, and Fidler]{ling2021editgan}
Huan Ling, Karsten Kreis, Daiqing Li, Seung~Wook Kim, Antonio Torralba, and Sanja Fidler.
\newblock Editgan: High-precision semantic image editing.
\newblock \emph{Advances in Neural Information Processing Systems}, 34:\penalty0 16331--16345, 2021.

\bibitem[Liu et~al.(2023)Liu, Yu, Zhang, Xu, Lei, Lai, Gu, Ding, Men, Yang, et~al.]{liu2023agentbench}
Xiao Liu, Hao Yu, Hanchen Zhang, Yanzhe Xu, Xinmei Lei, Huolin Lai, Yida Gu, Haohan Ding, Kai Men, Kai Yang, et~al.
\newblock Agentbench: Evaluating llms as agents.
\newblock \emph{arXiv preprint arXiv:2308.03688}, 2023.

\bibitem[Madry et~al.(2017)Madry, Makelov, Schmidt, Tsipras, and Vladu]{madry2017towards}
Aleksander Madry, Aleksandar Makelov, Ludwig Schmidt, Dimitris Tsipras, and Adrian Vladu.
\newblock Towards deep learning models resistant to adversarial attacks.
\newblock \emph{arXiv preprint arXiv:1706.06083}, 2017.

\bibitem[Mangaokar et~al.(2024)Mangaokar, Hooda, Choi, Chandrashekaran, Fawaz, Jha, and Prakash]{mangaokar2024prp}
Neil Mangaokar, Avijit Hooda, Jiahao Choi, Sidharth Chandrashekaran, Kassem Fawaz, Somesh Jha, and Atul Prakash.
\newblock Prp: Propagating universal perturbations to attack large language model guard-rails.
\newblock \emph{arXiv preprint arXiv:2402.15911}, 2024.

\bibitem[Marchant(2011)]{marchant2011growing}
Gary~E Marchant.
\newblock \emph{{The growing gap between emerging technologies and the law}}.
\newblock Springer, 2011.

\bibitem[Mazeika et~al.(2024)Mazeika, Phan, Yin, Zou, Wang, Mu, Sakhaee, Li, Basart, Li, et~al.]{mazeika2024harmbench}
Mantas Mazeika, Long Phan, Xuwang Yin, Andy Zou, Zifan Wang, Norman Mu, Elham Sakhaee, Nathaniel Li, Steven Basart, Bo~Li, et~al.
\newblock Harmbench: A standardized evaluation framework for automated red teaming and robust refusal.
\newblock \emph{arXiv preprint arXiv:2402.04249}, 2024.

\bibitem[Mehrotra et~al.(2023)Mehrotra, Zampetakis, Kassianik, Nelson, Anderson, Singer, and Karbasi]{mehrotra2023treeOfAttacks}
Anay Mehrotra, Manolis Zampetakis, Paul Kassianik, Blaine Nelson, Hyrum Anderson, Yaron Singer, and Amin Karbasi.
\newblock Tree of attacks: Jailbreaking black-box llms automatically, 2023.

\bibitem[Mei et~al.(2024{\natexlab{a}})Mei, Liu, Wang, Bi, and Chen]{mei2024slang}
Lingrui Mei, Shenghua Liu, Yiwei Wang, Baolong Bi, and Xueqi Chen.
\newblock Slang: New concept comprehension of large language models.
\newblock \emph{arXiv preprint arXiv:2401.12585}, 2024{\natexlab{a}}.

\bibitem[Mei et~al.(2024{\natexlab{b}})Mei, Liu, Wang, Bi, Mao, and Cheng]{mei2024not}
Lingrui Mei, Shenghua Liu, Yiwei Wang, Baolong Bi, Jiayi Mao, and Xueqi Cheng.
\newblock " not aligned" is not" malicious": Being careful about hallucinations of large language models' jailbreak.
\newblock \emph{arXiv preprint arXiv:2406.11668}, 2024{\natexlab{b}}.

\bibitem[Meng et~al.(2022{\natexlab{a}})Meng, Bau, Andonian, and Belinkov]{meng2022locating}
Kevin Meng, David Bau, Alex Andonian, and Yonatan Belinkov.
\newblock Locating and editing factual associations in {GPT}.
\newblock \emph{Advances in Neural Information Processing Systems}, 35, 2022{\natexlab{a}}.

\bibitem[Meng et~al.(2022{\natexlab{b}})Meng, Sharma, Andonian, Belinkov, and Bau]{meng2022mass}
Kevin Meng, Arnav~S Sharma, Alex Andonian, Yonatan Belinkov, and David Bau.
\newblock Mass-editing memory in a transformer.
\newblock \emph{arXiv preprint arXiv:2210.07229}, 2022{\natexlab{b}}.

\bibitem[Mikolov et~al.(2013)Mikolov, Chen, Corrado, and Dean]{mikolov2013efficient}
Tomas Mikolov, Kai Chen, Greg Corrado, and Jeffrey Dean.
\newblock Efficient estimation of word representations in vector space.
\newblock \emph{arXiv preprint arXiv:1301.3781}, 2013.

\bibitem[Mitchell et~al.(2021)Mitchell, Lin, Bosselut, Finn, and Manning]{mitchell2021fast}
Eric Mitchell, Charles Lin, Antoine Bosselut, Chelsea Finn, and Christopher~D Manning.
\newblock Fast model editing at scale.
\newblock \emph{arXiv preprint arXiv:2110.11309}, 2021.

\bibitem[Mitchell et~al.(2022)Mitchell, Lin, Bosselut, Manning, and Finn]{mitchell2022memory}
Eric Mitchell, Charles Lin, Antoine Bosselut, Christopher~D Manning, and Chelsea Finn.
\newblock Memory-based model editing at scale.
\newblock In \emph{International Conference on Machine Learning}, pp.\  15817--15831. PMLR, 2022.

\bibitem[Mowshowitz(2022)]{mowshowitz2022jailbreaking}
Zachary Mowshowitz.
\newblock Jailbreaking chatgpt on release day.
\newblock \url{https://www.lesswrong.com/posts/RYcoJdvmoBbi5Nax7/jailbreaking-chatgpt-on-release-day}, 2022.
\newblock Accessed: 2024-05-19.

\bibitem[OpenAI(2022)]{chatgpt}
OpenAI.
\newblock large-scale generative pre-training model for conversation.
\newblock \emph{OpenAI blog}, 2022.
\newblock URL \url{https://openai.com/blog/chatgpt}.

\bibitem[OpenAI(2023)]{openai2023gpt4}
OpenAI.
\newblock Gpt-4 technical report, 2023.

\bibitem[Ouyang et~al.(2022)Ouyang, Wu, Jiang, Almeida, Wainwright, Mishkin, Zhang, Agarwal, Slama, Ray, et~al.]{ouyang2022training}
Long Ouyang, Jeff Wu, Xu~Jiang, Diogo Almeida, Carroll Wainwright, Pamela Mishkin, Chong Zhang, Sandhini Agarwal, Katarina Slama, Alex Ray, et~al.
\newblock Training language models to follow instructions with human feedback.
\newblock \emph{Advances in Neural Information Processing Systems}, 35:\penalty0 27730--27744, 2022.

\bibitem[Perez et~al.(2022)Perez, Huang, Song, Cai, Ring, Aslanides, Glaese, McAleese, and Irving]{perez2022red}
Ethan Perez, Saffron Huang, Francis Song, Trevor Cai, Roman Ring, John Aslanides, Amelia Glaese, Nat McAleese, and Geoffrey Irving.
\newblock Red teaming language models with language models.
\newblock \emph{arXiv preprint arXiv:2202.03286}, 2022.

\bibitem[Rafailov et~al.(2023)Rafailov, Sharma, Mitchell, Ermon, Manning, and Finn]{rafailov2023direct}
Rafael Rafailov, Archit Sharma, Eric Mitchell, Stefano Ermon, Christopher~D Manning, and Chelsea Finn.
\newblock Direct preference optimization: Your language model is secretly a reward model.
\newblock \emph{arXiv preprint arXiv:2305.18290}, 2023.

\bibitem[Raji(2021)]{raji2021anatomy}
Inioluwa~Deborah Raji.
\newblock {The Anatomy of AI Audits: Form, Process, and Consequences}.
\newblock \emph{The Oxford Handbook of AI Governance}, 2021.
\newblock URL \url{https://doi.org/10.1093/oxfordhb/9780197579329.013.28}.

\bibitem[Robey et~al.(2023)Robey, Wong, Hassani, and Pappas]{robey2023smoothllm}
Alexander Robey, Eric Wong, Hamed Hassani, and George~J Pappas.
\newblock Smoothllm: Defending large language models against jailbreaking attacks.
\newblock \emph{arXiv preprint arXiv:2310.03684}, 2023.

\bibitem[Schlarmann \& Hein(2023)Schlarmann and Hein]{schlarmann2023adversarial}
Christian Schlarmann and Matthias Hein.
\newblock On the adversarial robustness of multi-modal foundation models.
\newblock In \emph{Proceedings of the IEEE/CVF International Conference on Computer Vision}, pp.\  3677--3685, 2023.

\bibitem[Shavit et~al.(2023)Shavit, Agarwal, Brundage, Adler, O'Keefe, Campbell, Lee, Mishkin, Eloundou, Hickey, et~al.]{shavitpractices}
Yonadav Shavit, Sandhini Agarwal, Miles Brundage, Steven Adler, Cullen O'Keefe, Rosie Campbell, Teddy Lee, Pamela Mishkin, Tyna Eloundou, Alan Hickey, et~al.
\newblock {Practices for Governing Agentic AI Systems}, 2023.

\bibitem[Shayegani et~al.(2023)Shayegani, Mamun, Fu, Zaree, Dong, and Abu-Ghazaleh]{shayegani2023survey}
Erfan Shayegani, Md~Abdullah~Al Mamun, Yu~Fu, Pedram Zaree, Yue Dong, and Nael Abu-Ghazaleh.
\newblock Survey of vulnerabilities in large language models revealed by adversarial attacks.
\newblock \emph{arXiv preprint arXiv:2310.10844}, 2023.

\bibitem[Song et~al.(2024)Song, Yuan, and Yang]{song2024fmint}
Zezheng Song, Jiaxin Yuan, and Haizhao Yang.
\newblock Fmint: Bridging human designed and data pretrained models for differential equation foundation model.
\newblock \emph{arXiv preprint arXiv:2404.14688}, 2024.

\bibitem[Touvron et~al.(2023{\natexlab{a}})Touvron, Lavril, Izacard, Martinet, Lachaux, Lacroix, Rozi{\`{e}}re, Goyal, Hambro, Azhar, Rodriguez, Joulin, Grave, and Lample]{DBLP:journals/corr/abs-2302-13971}
Hugo Touvron, Thibaut Lavril, Gautier Izacard, Xavier Martinet, Marie{-}Anne Lachaux, Timoth{\'{e}}e Lacroix, Baptiste Rozi{\`{e}}re, Naman Goyal, Eric Hambro, Faisal Azhar, Aur{\'{e}}lien Rodriguez, Armand Joulin, Edouard Grave, and Guillaume Lample.
\newblock Llama: Open and efficient foundation language models.
\newblock \emph{CoRR}, abs/2302.13971, 2023{\natexlab{a}}.
\newblock \doi{10.48550/ARXIV.2302.13971}.
\newblock URL \url{https://doi.org/10.48550/arXiv.2302.13971}.

\bibitem[Touvron et~al.(2023{\natexlab{b}})Touvron, Martin, Stone, Albert, Almahairi, Babaei, Bashlykov, Batra, Bhargava, Bhosale, Bikel, Blecher, Ferrer, Chen, Cucurull, Esiobu, Fernandes, Fu, Fu, Fuller, Gao, Goswami, Goyal, Hartshorn, Hosseini, Hou, Inan, Kardas, Kerkez, Khabsa, Kloumann, Korenev, Koura, Lachaux, Lavril, Lee, Liskovich, Lu, Mao, Martinet, Mihaylov, Mishra, Molybog, Nie, Poulton, Reizenstein, Rungta, Saladi, Schelten, Silva, Smith, Subramanian, Tan, Tang, Taylor, Williams, Kuan, Xu, Yan, Zarov, Zhang, Fan, Kambadur, Narang, Rodriguez, Stojnic, Edunov, and Scialom]{touvron2023llama}
Hugo Touvron, Louis Martin, Kevin Stone, Peter Albert, Amjad Almahairi, Yasmine Babaei, Nikolay Bashlykov, Soumya Batra, Prajjwal Bhargava, Shruti Bhosale, Dan Bikel, Lukas Blecher, Cristian~Canton Ferrer, Moya Chen, Guillem Cucurull, David Esiobu, Jude Fernandes, Jeremy Fu, Wenyin Fu, Brian Fuller, Cynthia Gao, Vedanuj Goswami, Naman Goyal, Anthony Hartshorn, Saghar Hosseini, Rui Hou, Hakan Inan, Marcin Kardas, Viktor Kerkez, Madian Khabsa, Isabel Kloumann, Artem Korenev, Punit~Singh Koura, Marie-Anne Lachaux, Thibaut Lavril, Jenya Lee, Diana Liskovich, Yinghai Lu, Yuning Mao, Xavier Martinet, Todor Mihaylov, Pushkar Mishra, Igor Molybog, Yixin Nie, Andrew Poulton, Jeremy Reizenstein, Rashi Rungta, Kalyan Saladi, Alan Schelten, Ruan Silva, Eric~Michael Smith, Ranjan Subramanian, Xiaoqing~Ellen Tan, Binh Tang, Ross Taylor, Adina Williams, Jian~Xiang Kuan, Puxin Xu, Zheng Yan, Iliyan Zarov, Yuchen Zhang, Angela Fan, Melanie Kambadur, Sharan Narang, Aurelien Rodriguez, Robert Stojnic, Sergey Edunov, and Thomas
  Scialom.
\newblock Llama 2: Open foundation and fine-tuned chat models, 2023{\natexlab{b}}.

\bibitem[Turner et~al.(2023)Turner, Thiergart, Udell, Leech, Mini, and MacDiarmid]{turner2023activation}
Alexander Turner, Laurence Thiergart, Madeleine Udell, Geoff Leech, Umac Mini, and Macklin MacDiarmid.
\newblock Activation addition: Steering language models without optimization.
\newblock \emph{arXiv preprint arXiv:2308.10248}, 2023.

\bibitem[Tutt(2017)]{tutt2017fda}
Andrew Tutt.
\newblock {An FDA for algorithms}.
\newblock \emph{Admin. L. Rev.}, 69:\penalty0 83, 2017.

\bibitem[Upchurch et~al.(2017)Upchurch, Gardner, Pleiss, Pless, Snavely, Bala, and Weinberger]{upchurch2017deep}
Paul Upchurch, Jacob Gardner, Geoff Pleiss, Robert Pless, Noah Snavely, Kavita Bala, and Kilian Weinberger.
\newblock Deep feature interpolation for image content changes.
\newblock In \emph{Proceedings of the IEEE Conference on Computer Vision and Pattern Recognition (CVPR)}, July 2017.

\bibitem[Vaswani(2017)]{vaswani2017attention}
A~Vaswani.
\newblock Attention is all you need.
\newblock \emph{Advances in Neural Information Processing Systems}, 2017.

\bibitem[Veale et~al.(2023)Veale, Matus, and Gorwa]{veale2023ai}
Michael Veale, Kira Matus, and Robert Gorwa.
\newblock {AI and Global Governance: Modalities, Rationales, Tensions}.
\newblock \emph{Annual Review of Law and Social Science}, 19, 2023.

\bibitem[Vega et~al.(2023)Vega, Chaudhary, Xu, and Singh]{vega2023bypassing}
Juan Vega, Irshad Chaudhary, Cuiying Xu, and Gagandeep Singh.
\newblock Bypassing the safety training of open-source llms with priming attacks.
\newblock \emph{arXiv preprint arXiv:2312.12321}, 2023.

\bibitem[Wang et~al.(2024{\natexlab{a}})Wang, Chen, Peng, and Chang]{wang2024frustratingly}
Yiwei Wang, Muhao Chen, Nanyun Peng, and Kai-Wei Chang.
\newblock Frustratingly easy jailbreak of large language models via output prefix attacks.
\newblock 2024{\natexlab{a}}.

\bibitem[Wang et~al.(2024{\natexlab{b}})Wang, Ma, Zhang, Ni, Chandra, Guo, Ren, Arulraj, He, Jiang, et~al.]{wang2024mmlu}
Yubo Wang, Xueguang Ma, Ge~Zhang, Yuansheng Ni, Abhranil Chandra, Shiguang Guo, Weiming Ren, Aaran Arulraj, Xuan He, Ziyan Jiang, et~al.
\newblock Mmlu-pro: A more robust and challenging multi-task language understanding benchmark.
\newblock \emph{arXiv preprint arXiv:2406.01574}, 2024{\natexlab{b}}.

\bibitem[Wei et~al.(2024)Wei, Haghtalab, and Steinhardt]{wei2024jailbroken}
Alexander Wei, Nika Haghtalab, and Jacob Steinhardt.
\newblock Jailbroken: How does llm safety training fail?
\newblock \emph{Advances in Neural Information Processing Systems}, 36, 2024.

\bibitem[Wei et~al.(2023)Wei, Haghtalab, and Steinhardt]{wei2023jailbroken}
Jason Wei, Nika Haghtalab, and Jacob Steinhardt.
\newblock Jailbroken: How does llm safety training fail?
\newblock \emph{arXiv preprint arXiv:2307.02483}, 2023.

\bibitem[Wen et~al.(2024)Wen, Jain, Kirchenbauer, Goldblum, Geiping, and Goldstein]{wen2024hard}
Yuxin Wen, Neel Jain, John Kirchenbauer, Micah Goldblum, Jonas Geiping, and Tom Goldstein.
\newblock Hard prompts made easy: Gradient-based discrete optimization for prompt tuning and discovery.
\newblock \emph{Advances in Neural Information Processing Systems}, 36, 2024.

\bibitem[Yao et~al.(2023)Yao, Wang, Tian, Cheng, Li, Deng, Chen, and Zhang]{yao2023editing}
Yunzhi Yao, Peng Wang, Bozhong Tian, Siyuan Cheng, Zhoubo Li, Shumin Deng, Huajun Chen, and Ningyu Zhang.
\newblock Editing large language models: Problems, methods, and opportunities.
\newblock \emph{arXiv preprint arXiv:2305.13172}, 2023.

\bibitem[Yuan et~al.(2024)Yuan, Jiao, Wang, Huang, Xu, Liang, He, and Tu]{yuan2024refuse}
Youliang Yuan, Wenxiang Jiao, Wenxuan Wang, Jen-tse Huang, Jiahao Xu, Tian Liang, Pinjia He, and Zhaopeng Tu.
\newblock Refuse whenever you feel unsafe: Improving safety in llms via decoupled refusal training.
\newblock \emph{arXiv preprint arXiv:2407.09121}, 2024.

\bibitem[Zhang et~al.(2024{\natexlab{a}})Zhang, Yu, Li, Dong, Su, Chu, and Yu]{zhang2024mm}
Duzhen Zhang, Yahan Yu, Chenxing Li, Jiahua Dong, Dan Su, Chenhui Chu, and Dong Yu.
\newblock Mm-llms: Recent advances in multimodal large language models.
\newblock \emph{arXiv preprint arXiv:2401.13601}, 2024{\natexlab{a}}.

\bibitem[Zhang et~al.(2024{\natexlab{b}})Zhang, Li, Zhang, Yin, Liu, and Moshfeghi]{zhang2024geoeval}
Jiaxin Zhang, Zhongzhi Li, Mingliang Zhang, Fei Yin, Chenglin Liu, and Yashar Moshfeghi.
\newblock Geoeval: benchmark for evaluating llms and multi-modal models on geometry problem-solving.
\newblock \emph{arXiv preprint arXiv:2402.10104}, 2024{\natexlab{b}}.

\bibitem[Zhang et~al.(2023)Zhang, Li, Yin, and Liu]{zhang2023lans}
Ming-Liang Zhang, Zhong-Zhi Li, Fei Yin, and Cheng-Lin Liu.
\newblock Lans: A layout-aware neural solver for plane geometry problem.
\newblock \emph{arXiv preprint arXiv:2311.16476}, 2023.

\bibitem[Zhang et~al.(2024{\natexlab{c}})Zhang, Li, Yin, Lin, and Liu]{zhang2024fuse}
Ming-Liang Zhang, Zhong-Zhi Li, Fei Yin, Liang Lin, and Cheng-Lin Liu.
\newblock Fuse, reason and verify: Geometry problem solving with parsed clauses from diagram.
\newblock \emph{arXiv preprint arXiv:2407.07327}, 2024{\natexlab{c}}.

\bibitem[Zheng et~al.(2023)Zheng, Chiang, Sheng, Zhuang, Wu, Zhuang, Lin, Li, Li, Xing, et~al.]{zheng2023mtbench}
Lianmin Zheng, Wei-Lin Chiang, Ying Sheng, Siyuan Zhuang, Zhanghao Wu, Yonghao Zhuang, Zi~Lin, Zhuohan Li, Dacheng Li, Eric Xing, et~al.
\newblock Judging llm-as-a-judge with mt-bench and chatbot arena.
\newblock \emph{arXiv preprint arXiv:2306.05685}, 2023.

\bibitem[Zhou et~al.(2024)Zhou, Li, and Wang]{zhou2024robust}
Andy Zhou, Boyi Li, and Haizhong Wang.
\newblock Robust prompt optimization for defending language models against jailbreaking attacks.
\newblock \emph{arXiv preprint arXiv:2401.17263}, 2024.

\bibitem[Zou et~al.(2023{\natexlab{a}})Zou, Phan, Chen, Campbell, Guo, Ren, Pan, Yin, Mazeika, Dombrowski, et~al.]{zou2023representation}
Andy Zou, Long Phan, Sarah Chen, Jesse Campbell, Phillip Guo, Rishi Ren, Alexander Pan, Xuwang Yin, Mantas Mazeika, Anne-Kathrin Dombrowski, et~al.
\newblock Representation engineering: A top-down approach to ai transparency.
\newblock \emph{arXiv preprint arXiv:2310.01405}, 2023{\natexlab{a}}.

\bibitem[Zou et~al.(2023{\natexlab{b}})Zou, Wang, Kolter, and Fredrikson]{zou2023universal}
Andy Zou, Zifan Wang, J.~Zico Kolter, and Matt Fredrikson.
\newblock Universal and transferable adversarial attacks on aligned language models, 2023{\natexlab{b}}.

\bibitem[Zou et~al.(2024)Zou, Phan, Wang, Duenas, Lin, Andriushchenko, Wang, Kolter, Fredrikson, and Hendrycks]{zou2024improving}
Andy Zou, Long Phan, Justin Wang, Derek Duenas, Maxwell Lin, Maksym Andriushchenko, Rowan Wang, Zico Kolter, Matt Fredrikson, and Dan Hendrycks.
\newblock Improving alignment and robustness with circuit breakers.
\newblock \emph{arXiv preprint arXiv}, 2406, 2024.

\end{thebibliography}
\bibliographystyle{iclr2025_conference}

\appendix
\section{Related Work}
\textbf{Adversarial Attacks on LLMs.} Numerous manually crafted attack prompts have exposed vulnerabilities in modern LLMs \citep{mowshowitz2022jailbreaking, wei2023jailbroken}, forming the foundation for \textit{red teaming} efforts for frontier models \citep{openai2023gpt4, claude3_report}. However, the process of red teaming lacks standardization across different models \citep{feffer2024red}, making it difficult to compare the effectiveness of safety interventions across various platforms. Automated red teaming approaches have shown promising results \citep{perez2022red, chao2023jailbreaking}. Of particular note are transfer attacks using adversarial suffixes, optimized via gradients \citep{zou2023universal}. White-box attacks, such as prefilling attacks, exploit internal model structures to elicit harmful outputs \citep{vega2023bypassing}. Recent efforts to consolidate and evaluate these methods can be found in HarmBench \citep{mazeika2024harmbench} and \textsc{BabyBLUE}~\cite{mei2024not}. In the multi-modal domain, attacks span from simple typographic manipulations to sophisticated gradient-based optimizations \citep{carlini2023aligned, bailey2023image}. While some benchmarks exist for LLM-based agents \citep{liu2023agentbench}, the exploration of their safety and robustness remains in its infancy.

\textbf{Defenses for LLMs.} Common defenses, such as Reinforcement Learning from Human Feedback (RLHF) \citep{christiano2017deep} and Direct Preference Optimization (DPO) \citep{rafailov2023direct}, rely heavily on human annotations. However, they often fail against sophisticated adversarial attacks \citep{zou2023universal}. More robust methods, such as prompt optimization to reject harmful content \citep{zhou2024robust}, show potential but are limited in generalizability. Adversarial training, a strategy derived from computer vision \citep{madry2017towards}, has been applied to LLMs but is computationally demanding and causes performance drops in general benchmarks \citep{zheng2023mtbench}. Inference-time defenses, such as perplexity filters \citep{alon2023detecting}, are only effective against static, non-adaptive attacks. More advanced approaches, such as erase-and-check strategies \citep{robey2023smoothllm}, incur significant computational costs. System-level defenses also remain vulnerable to well-designed adversarial inputs \citep{mangaokar2024prp}. In contrast, our approach introduces circuit breakers, inspired by advances in representation engineering \citep{zou2024improving}, which dynamically interrupt harmful output generation. This method is computationally efficient, bypassing the limitations of refusal and adversarial training by directly manipulating representations responsible for harmful content. It applies to both unimodal and multimodal LLMs, preventing harmful output without degrading the model’s utility. Additionally, Decoupled Refusal Training (DeRTa) \citep{yuan2024refuse} addresses refusal position bias in safety tuning data, ensuring LLMs can reject harmful prompts at any point in the response sequence. This novel approach significantly enhances safety by equipping models with the ability to transition from harmful to safe responses dynamically.

\textbf{Representation Engineering.}
As contemporary defense strategies that solely supervise model outputs often fall short in achieving the necessary levels of controllability and reliability, there has been a growing interest in techniques that analyze and manage the internal representations of models. Representation engineering encompasses a broad range of research areas, including the discovery of emergent, interpretable structures within intermediate representations \citep{caron2021emerging, mikolov2013efficient, zou2023representation}, the identification and modification of embedded knowledge \citep{meng2022locating, meng2022mass, mitchell2021fast}, and the steering of model outputs \citep{bau2020semantic, ilharco2022editing, ling2021editgan, upchurch2017deep, turner2023activation}. A particularly relevant advancement in this field is the control vector baseline introduced by \cite{zou2023representation}, which enhances large language models' resilience against adversarial attacks. This approach not only utilizes control vectors but also incorporates representation-level loss functions to adjust internal representations effectively. Building on this foundation, recent developments have extended these methods to robustly unlearn harmful knowledge through a technique known as RMU \citep{li2024wmdp}, demonstrating the versatility of representation engineering in tackling more complex objectives. Despite previous attempts to eliminate harmful circuits using bottom-up mechanistic interpretability \citep{li2023cb}, these methods have proven inadequate. 

\textbf{Governance Challenges}
Effective governance is crucial for the safety and societal alignment of LLMs and broader AI systems \citep{bullock2022oxford, veale2023ai}. Current governance frameworks, including formal regulations, norms, soft law, and industry standards, are largely nascent and often voluntary, as seen in initiatives like the EU AI Act \citep{EUAIact2024}. Several meta-challenges impede the efficacy of LLM governance, such as the insufficient scientific understanding of LLMs and unreliable technical tools \citep{raji2021anatomy, guha2023ai, kapoor2024societal}, the slow and inflexible nature of existing governance institutions \citep{marchant2011growing, engler2023comprehensive}, and the significant influence of corporate power which raises risks of regulatory capture \citep{ai_safety_newsletter32, costanza2022audits}. Additionally, there is a pressing need for international cooperation and clearer accountability mechanisms \citep{dafoe2018ai, ideas2021, shavitpractices, barnard2024aigovernance}. Addressing these challenges requires innovative approaches, such as establishing new regulatory bodies, enhancing public-private partnerships while mitigating capture risks, and accelerating technical research to inform governance \citep{tutt2017fda, au2023china, hadfield2023regulatory}. Without overcoming these obstacles, ensuring that LLMs contribute positively to society while minimizing harm remains a significant concern.

\paragraph{Model Editing and Tuning}
Model editing is an effective approach for knowledge editing (KE), where the internal structure of the model is adjusted to alter its output for specific edited content. Recent model editing and tuning techniques for LLMs~\citep{meng2022locating, meng2022mass, mitchell2022memory, yao2023editing, bi2024lpnl} commonly involve either integrating an auxiliary network with the original model or modifying and adding parameters to steer the model’s responses. In-Context Editing (ICE)\citep{bi2024adaptive, bi2024decoding, bi2024factualitydecodingfreelunch, bi2024struedit} and In-Context Understanding\citep{mei2024slang} show promise by allowing edits to LLMs through prompting with modified facts and retrieving relevant editing demonstrations from a memory of edits. Moreover, models are demonstrating powerful problem-solving capabilities across an increasing number of domains~\citep{zhang2023lans, zhang2024geoeval, zhang2024fuse, li2024cmmath}.

\section{Notation and Definitions}

In this section, we provide definitions for all symbols and variables used throughout the paper, define key concepts such as \emph{Representation Collapse}, and explicitly state the assumptions underlying our theoretical results.

\subsection{Notation}

\begin{itemize}
    \item $\mathcal{M} = (f_\theta, \mathcal{X}, \mathcal{Y})$: The language model, where $f_\theta: \mathcal{X} \rightarrow \mathcal{Y}$ is the model function with parameters $\theta \in \Theta$, $\mathcal{X}$ is the input space, and $\mathcal{Y}$ is the output space.
    \item $\theta \in \Theta$: Parameters of the language model.
    \item $\mathcal{X}$: Input space (set of all possible inputs).
    \item $\mathcal{Y}$: Output space (set of all possible outputs).
    \item $\mathcal{D}$: Data distribution over which expectations are taken.
    \item $\mathcal{D}_{\text{benign}}$: Distribution of benign data samples.
    \item $\mathcal{D}_{\text{adversarial}}$: Distribution of adversarial or harmful data samples.
    \item $\mathcal{L}_{\text{benign}}: \mathcal{Y} \times \mathcal{Y} \rightarrow \mathbb{R}_{\geq 0}$: Loss function ensuring accuracy on benign data.
    \item $\mathcal{L}_{\text{adv}}: \mathcal{Y} \rightarrow \mathbb{R}_{\geq 0}$: Loss function penalizing adversarial or harmful outputs.
    \item $\lambda \in \mathbb{R}_{> 0}$: Regularization factor balancing between benign and adversarial losses.
    \item $f_{\theta^*}$: The optimized model after training.
    \item $\mathcal{L}_{\text{safety}}: \mathcal{Y} \rightarrow \mathbb{R}_{\geq 0}$: Safety-oriented loss function.
    \item $\mathcal{L}_{\text{utility}}: \mathcal{Y} \rightarrow \mathbb{R}_{\geq 0}$: Utility loss function measuring the usefulness of the output.
    \item $\text{rep}_M: \mathcal{X} \rightarrow \mathbb{R}^d$: Function mapping inputs to $d$-dimensional internal representations.
    \item $\mathcal{L}_{\text{mod}}: \mathbb{R}^d \rightarrow \mathbb{R}_{\geq 0}$: Loss function enforcing constraints on harmful input representations.
    \item $\mathcal{R}_\phi: \mathbb{R}^d \rightarrow [0,1]$: Token-level router function parameterized by $\phi$.
    \item $\sigma: \mathbb{R} \rightarrow [0,1]$: Sigmoid activation function.
    \item $r_i = \sigma(\mathcal{R}_\phi(z_i))$: Harmfulness score for token $t_i$.
    \item $z_i \in \mathbb{R}^d$: Vector representation of token $t_i$.
    \item $N$: Sequence length (number of tokens in the input).
    \item $\mathcal{S} = \{s_1, \ldots, s_M\}$: Set of sentences in a sequence.
    \item $s_j$: The $j$-th sentence in the sequence.
    \item $K_j$: Number of tokens in sentence $s_j$.
    \item $\mathcal{C}: \mathcal{P}(\mathcal{T}) \rightarrow \{0,1\}$: Ideal contextual harmfulness classifier over the power set of all possible tokens $\mathcal{P}(\mathcal{T})$.
    \item $T_s \subseteq \{t_1, \ldots, t_N\}$: Subset of tokens.
    \item $g: [0,1]^{K_j} \rightarrow [0,1]$: Aggregation function over token harmfulness scores in a sentence.
    \item $h: [0,1]^{|T_s|} \rightarrow \{0,1\}$: Aggregation function over token harmfulness scores in a subset $T_s$.
    \item $\mathcal{A}_\psi: \mathcal{X} \rightarrow \mathbb{R}^k$: LoRA-based activator function parameterized by $\psi$.
    \item $\mathcal{M}: \mathcal{X} \rightarrow \mathcal{Y}$: The fixed, pre-trained language model.
    \item $h_i$: Hidden state of the model at token position $i$.
    \item $\mathcal{L}_{\text{token}}: \mathcal{T} \times [0,1] \rightarrow \mathbb{R}_{\geq 0}$: Token-level loss function.
    \item $\mathcal{L}_{\text{global}}: \mathcal{X} \times \mathcal{Y} \rightarrow \mathbb{R}_{\geq 0}$: Global coherence loss function.
    \item $\mathcal{L}_{\text{AR}}$: Adversarial Regularization Loss, used to encourage activators to produce higher activation signals for adversarial inputs.
    \item $\mathcal{L}_{\text{retain}}$: Retention Loss, used to ensure that activators do not interfere with the representations of benign inputs.
    \item $\mathcal{L}_{\text{signal}}$: Signal Vector Learning Loss, used to learn signal vectors that produce low activation signals for benign inputs and high activation signals for adversarial inputs.
    \item $x^{+} \sim \mathcal{D}_{\text{benign}}$: Benign input samples from the benign data distribution.
    \item $x^{-} \sim \mathcal{D}_{\text{adversarial}}$: Adversarial input samples from the adversarial data distribution.
    \item $N_{\text{act}}$: Number of activators.
    \item $c_{\text{AR}}(t)$: Time-dependent coefficient for Adversarial Regularization Loss at training step $t$.
    \item $c_{\text{retain}}(t)$: Time-dependent coefficient for Retention Loss at training step $t$.
    \item $k$: Context window size, determining the number of surrounding tokens considered by the router network.
    \item $\mathcal{L}_{\text{router}}$: Loss function for the router network, used to train the router for fine-grained token-level harmfulness classification.
    \item $\gamma$: Focusing parameter used in the focal loss to address class imbalance.
    \item $\alpha$: Coefficient used in the activator training loss scheduling.
\end{itemize}

\subsection{Explanation}

\begin{itemize}
    \item \textbf{Time-dependent Coefficients} ($c_{\text{AR}}(t)$ and $c_{\text{retain}}(t)$): These coefficients dynamically adjust during training to balance the Adversarial Regularization Loss and Retention Loss. Defining these symbols clarifies the mechanism for weighting different loss components throughout the training process.
    
    \item \textbf{Context Window Size} ($k$): In the router network, the context window size determines how many surrounding tokens are considered for each token's harmfulness assessment. Clearly defining this parameter helps in understanding the scope of contextual information the router utilizes.
    
    \item \textbf{Loss Functions} ($\mathcal{L}_{\text{router}}$ and $\mathcal{L}_{\text{signal}}$): Defining these loss functions provides a comprehensive description of the different training objectives and optimization directions within the model.
    
    \item \textbf{Focusing Parameter} ($\gamma$) and \textbf{Coefficient} ($\alpha$): These hyperparameters play crucial roles in the loss functions, and defining them helps readers understand the mechanisms for adjusting the influence of different loss components.
\end{itemize}

\subsection{Definitions}

\begin{definition}[Representation Collapse]\label{def:representation_collapse}
\emph{Representation Collapse} refers to the phenomenon where the internal representations of distinct inputs become nearly identical due to over-regularization or excessive constraints imposed during training. Formally, for a model $M$ with representation function $\text{rep}_M: \mathcal{X} \rightarrow \mathbb{R}^d$, representation collapse occurs when:
\[
\|\text{rep}_M(x_1) - \text{rep}_M(x_2)\|_2 < \epsilon, \quad \forall x_1, x_2 \in \mathcal{X}_{\text{adversarial}},
\]
where $\mathcal{X}_{\text{adversarial}} \subseteq \mathcal{X}$ is the set of adversarial inputs, and $\epsilon$ is a small positive constant. This collapse reduces the model's ability to distinguish between different adversarial inputs, potentially impacting its overall performance and expressiveness.
\end{definition}

\begin{definition}[Gradient Masking]\label{def:gradient_masking}
\emph{Gradient Masking} is a situation where the gradients of the loss function with respect to the input are near zero, giving a false sense of security against adversarial attacks. Formally, for an input $x' \in \mathcal{X}$:
\[
\|\nabla_x \mathcal{L}_{\text{adv}}(f_\theta(x'))\|_2 \approx 0, \quad \text{but} \quad \|f_\theta(x' + \delta) - f_\theta(x')\|_2 \gg 0,
\]
where $\delta$ is a small perturbation. This indicates that small changes in the input can still lead to significant differences in the output, despite minimal gradients.
\end{definition}

\subsection{Assumptions}

Throughout our theoretical analysis, we make the following assumptions:

\begin{enumerate}
    \item \textbf{Data Distribution:} The data distribution $\mathcal{D}$ is fixed, and samples are drawn independently and identically distributed (i.i.d.).
    \item \textbf{Model Capacity:} The language model $f_\theta$ has sufficient capacity to approximate the desired functions within the hypothesis space $\Theta$.
    \item \textbf{Loss Functions:} The loss functions $\mathcal{L}_{\text{benign}}$, $\mathcal{L}_{\text{adv}}$, $\mathcal{L}_{\text{safety}}$, $\mathcal{L}_{\text{utility}}$, $\mathcal{L}_{\text{mod}}$, $\mathcal{L}_{\text{token}}$, and $\mathcal{L}_{\text{global}}$ are convex and differentiable with respect to their arguments.
    \item \textbf{Regularization Parameter:} The regularization factor $\lambda$ is a positive constant that balances the trade-off between conflicting objectives.
    \item \textbf{Optimization Convergence:} The optimization procedures employed converge to a (local) minimum of the loss functions.
    \item \textbf{Ideal Functions:} The functions $\mathcal{C}$, $g$, and $h$ are considered idealized for theoretical analysis and may not be perfectly realizable in practice.
    \item \textbf{Activation Functions:} Activation functions such as the sigmoid $\sigma$ are smooth and monotonically increasing.
    \item \textbf{Router and Activator Functions:} The router $\mathcal{R}_\phi$ and activator $\mathcal{A}_\psi$ have sufficient capacity to model the necessary mappings for effective moderation.
\end{enumerate}
\subsection{Additional Assumptions}
\begin{assumption}[Robustness to Contextual Variations]
    The moderation function $\mathcal{R}$ maintains consistent performance across different contextual variations in the input data distribution, such that for any context $c$,
    \[
    \mathbb{P}(\mathcal{R}(h_i) = 1 \mid t_i \in \mathcal{T}_{\text{adv}}, c) \geq 1 - \delta,
    \]
    \[
    \mathbb{P}(\mathcal{R}(h_i) = 1 \mid t_i \notin \mathcal{T}_{\text{adv}}, c) \leq \epsilon,
    \]
    where $\delta, \epsilon \in (0,1)$ are small constants.
\end{assumption}

\section{Additional Experimental Details}
\label{appendix:experimental_details}
\subsection{Dataset}
\label{appndix:dataset}
\textbf{Redacted Circuit Breaker Dataset:}
The \textit{Redacted Circuit Breaker Dataset} is a refined version of the refusal-retain dataset from \citep{zou2024improving}, containing harmful content generated by various uncensored language models with precise annotations. Initial annotations were performed using GPT-4o to identify potentially harmful segments. These annotations were then refined through precise character-level Inside-Outside-Beginning (IOB) tagging to delineate harmful entities accurately. During preprocessing, character-level tags were converted into token-level labels to facilitate fine-grained moderation. The dataset comprises a total of 4,993 entries, with 3,994 allocated for training and 999 for testing. 

\textbf{Retain Dataset:}
The \textit{Retain Dataset} consists of two subsets:
\begin{itemize}
    \item \textit{UltraChat}: Contains benign queries and conversational exchanges designed to represent typical user interactions. 
    \item \textit{XSTest}: Includes exaggerated refusal examples that challenge the model's ability to handle extreme cases. 
    \item Additionally, we incorporate the \textit{chosen} subset from the \textit{Anthropic/hh-rlhf} dataset.
\end{itemize}
 This subset is sampled to ensure that the final \textit{Retain Dataset} matches the size of the \textit{Redacted Circuit Breaker Dataset}, with 3,994 entries used for training. This balanced approach ensures equitable contribution from both datasets during training, enhancing the model's ability to generate safe and informative responses while effectively moderating harmful content. 

\textbf{ASR Test Dataset:}
For the Adversarial Success Rate (ASR) test, we selected the top 200 behaviors from the HarmBench benchmark, focusing on those with the highest attack success rates using \textsc{BabyBLUE}~\citep{mei2024not} evaluators. This selection targets the most challenging adversarial conditions, enabling a rigorous evaluation of the model's robustness. 

\subsection{Setup}
\label{appendix:setup}
\textbf{Language Models:}
We conduct experiments using the following language models:
\begin{itemize}
    \item \textsc{LLaMA2-7B-Chat}
    \item \textsc{LLaMA3-8B-Instruct}
    \item \textsc{Mistral-7B-Instruct}
\end{itemize}

\textbf{Router and Activator Configuration:}
We deploy a single \router at the 30th layer of each model, utilizing low-rank matrices with a dimension of $r=64$. The choice of layer and rank dimension was based on preliminary experiments indicating optimal performance in balancing computational efficiency and moderation accuracy. The 30th layer was selected because it is closer to the later stages of the model, allowing \ours to capture more refined representations without sacrificing parallelism in the computation. By placing the \router at this layer, the core LLM architecture does not need to wait for the moderation results from \ours, ensuring that the main network can continue processing efficiently. This choice strikes a balance between leveraging rich, late-stage features and maintaining the overall inference speed.

The router network is configured with a transformer~\citep{vaswani2017attention} encoder, where the number of layers is set to 1, and it uses 2 attention heads and a feedforward dimension of 512. The input `$\texttt{hidden\_size}$' for the router matches the hidden size of the model itself, ensuring no further downsampling occurs, which allows the router to directly process the full-resolution representations. This design allows the router to preserve the detailed contextual information necessary for accurate token-level moderation. The router's final classification layer produces harmfulness scores for each token, enabling fine-grained detection and redaction of harmful content. Varying $r$ (the rank of low-rank matrices) impacts both the granularity of moderation and the computational overhead. Larger values of $r$ allow more nuanced token-level moderation but increase memory and computational costs, while smaller values reduce complexity but may miss subtle harmful content.

In our experiments, we only used a single activator, also located at the 30th layer, as it was found to be highly effective for the current limited adversarial dataset. The use of a single activator provided sufficient coverage for the moderation tasks at hand. However, as tasks become more complex and involve richer representation spaces, the number of activators can be increased to capture more nuanced patterns in the data and to manage more sophisticated adversarial scenarios.

\textbf{Hardware and Training Parameters:}
All experiments are conducted on 4 NVIDIA Tesla A800 GPUs, each equipped with 80 GB of memory. The training process for each epoch takes approximately 4 hours, allowing for sufficient convergence of the activators and router networks. Inference is performed with a maximum sequence length of 8192 tokens to accommodate complex prompts. We utilize a batch size of 8 for training and a batch size of 1 for evaluation across all experiments, optimizing for both computational efficiency and model performance. The model is trained for a total of 150 steps, with a learning rate of $1\times10^{-5}$ and weight decay set to 0.0. We employ a constant learning rate scheduler, with gradient accumulation steps set to 1 to maintain stability during training.

To ensure efficient use of GPU resources, we enabled mixed precision training with bf16, and gradient checkpointing was employed to reduce memory usage during backpropagation. The training also leveraged DeepSpeed configuration to further optimize distributed training. Logging was performed every 10 steps, and evaluation was triggered every 1000 steps, ensuring detailed tracking of performance metrics throughout training.

\subsection{Evaluation}
\label{appendix:eval}
\textbf{Redaction Accuracy:}
We assess redaction accuracy using the \textit{pass @ n\%} metric. This metric evaluates whether a continuous sequence of tokens requiring redaction is successfully redacted if at least \textit{n\%} of the sequence is redacted. This flexible measure is particularly effective for evaluating models on longer sequences of harmful content. In our experiments, we used $n=90$, as human annotator volunteers consistently agreed that if 90\% of a harmful sequence has been redacted, the remaining content can be considered sufficiently neutralized. This threshold strikes a balance between ensuring content safety and maintaining the informativeness of the model's output.

\textbf{Activator Performance:}
The activator component is deemed successful if it triggers within the first 10\% of harmful tokens in a given text sequence. This early detection criterion allows for proactive moderation, minimizing the generation of harmful content. 

\subsection{Red Teaming Method Descriptions}\label{sec:method_description}
\begin{itemize}
    \item \textit{Direct Request}: This approach employs the actual behavior statements as test inputs, assessing the model's capability to reject explicit requests for these behaviors, especially when such requests are unambiguous and often indicate malicious intent.
    \item \textit{GCG} \citep{zou2023universal}: This technique involves crafting an adversarial suffix at the token level, which is then appended to a user prompt to generate a test case. The optimization process is designed to increase the log probability that the target LLM will respond affirmatively, exhibiting the desired behavior.
    \item \textit{PEZ} \citep{wen2024hard}: Similar to GCG, PEZ optimizes an adversarial suffix at the token level but utilizes a straight-through estimator and nearest-neighbor projection to focus on hard tokens during optimization.
    \item \textit{TAP-Transfer} \citep{mehrotra2023treeOfAttacks}: An extension of the TAP method, TAP-Transfer employs GPT-4 as both the judge and target model, while using Mixtral 8x7B as the attack model. The test cases generated through this method are intended to be transferable to other models, and it is abbreviated as TAP-T.
    \item \textit{PAIR} \citep{chao2023jailbreaking}: This method involves the iterative prompting of an attacker LLM to explore and induce specific harmful behaviors from the target LLM, systematically probing the model for vulnerabilities.
\end{itemize}


\textbf{Overall Model Performance:}
To ensure that safety enhancements do not degrade the model's general capabilities, we evaluate overall performance on MMLU-Pro~\citep{wang2024mmlu} and MT-Bench~\citep{zheng2023mtbench}. These evaluations confirm that our moderation framework maintains a balance between safety and utility, ensuring that the model remains effective across a wide range of tasks. 
\section{Proofs and Additional Theorems}

\paragraph{Proof (Inherent Trade-off in Global Output-Level Optimization)}
Consider the language model $\mathcal{M} = (f_\theta, \mathcal{X}, \mathcal{Y})$ parameterized by $\theta \in \Theta$, where $f_\theta: \mathcal{X} \rightarrow \mathcal{Y}$ maps inputs to outputs. The global output-level optimization seeks to minimize the combined loss function:
\[
\theta^* = \arg\min_{\theta \in \Theta} \mathbb{E}_{x \sim \mathcal{D}_{\text{benign}}} \left[ \mathcal{L}_{\text{benign}}(f_\theta(x), y) \right] + \lambda \mathbb{E}_{x' \sim \mathcal{D}_{\text{adversarial}}} \left[ \mathcal{L}_{\text{adv}}(f_\theta(x')) \right],
\]
where:
\begin{itemize}
    \item $\mathcal{L}_{\text{benign}}: \mathcal{Y} \times \mathcal{Y} \rightarrow \mathbb{R}_{\geq 0}$ denotes the utility loss on benign inputs.
    \item $\mathcal{L}_{\text{adv}}: \mathcal{Y} \rightarrow \mathbb{R}_{\geq 0}$ denotes the safety loss on adversarial inputs.
    \item $\lambda > 0$ is a weighting factor balancing the two loss terms.
    \item $\mathcal{D}_{\text{benign}}$ and $\mathcal{D}_{\text{adversarial}}$ represent the distributions of benign and adversarial inputs, respectively.
\end{itemize}
Assume that $\mathcal{L}_{\text{benign}}$ and $\mathcal{L}_{\text{adv}}$ are not perfectly aligned. Specifically, there exists at least one benign input $x_b \in \mathcal{X}_{\text{benign}}$ such that optimizing $\mathcal{L}_{\text{adv}}$ increases $\mathcal{L}_{\text{benign}}$. Formally, for this $x_b$:
\[
\nabla_\theta \mathcal{L}_{\text{benign}}(f_\theta(x_b), y) \cdot \nabla_\theta \mathcal{L}_{\text{adv}}(f_\theta(x_b)) < 0.
\]
At the optimal parameter $\theta^*$, the gradient of the combined loss must satisfy:
\[
\nabla_\theta \left[ \mathbb{E}_{x \sim \mathcal{D}_{\text{benign}}} \mathcal{L}_{\text{benign}}(f_\theta(x), y) + \lambda \mathbb{E}_{x' \sim \mathcal{D}_{\text{adversarial}}} \mathcal{L}_{\text{adv}}(f_\theta(x')) \right] = 0.
\]
Focusing on the benign input $x_b$, we can derive:
\begin{align}
    \nabla_\theta \mathcal{L}_{\text{benign}}(f_\theta(x_b), y) + \lambda \nabla_\theta \mathcal{L}_{\text{adv}}(f_\theta(x_b)) &= 0 \\
    \|\nabla_\theta \mathcal{L}_{\text{benign}}(f_\theta(x_b), y)\|_2^2 + \lambda \nabla_\theta \mathcal{L}_{\text{adv}}(f_\theta(x_b)) \cdot \nabla_\theta \mathcal{L}_{\text{benign}}(f_\theta(x_b), y) &= 0 \\
     -\lambda \nabla_\theta \mathcal{L}_{\text{adv}}(f_\theta(x_b)) \cdot \nabla_\theta \mathcal{L}_{\text{benign}}(f_\theta(x_b), y) &= \|\nabla_\theta \mathcal{L}_{\text{benign}}(f_\theta(x_b), y)\|_2^2 \\
    &> 0
\end{align}
This implies:
\[
\mathcal{L}_{\text{benign}}(f_{\theta^*}(x_b), y) > \mathcal{L}_{\text{benign}}(f_\theta(x_b), y),
\]
demonstrating that the optimized model $\theta^*$ incurs a higher utility loss on the benign input $x_b$ compared to the original model $\theta$. Thus, an inherent trade-off exists in global output-level optimization between minimizing safety loss and preserving utility.

\begin{theorem}[Information Preservation]
\label{thm:information_preservation}
The {\ours} framework preserves mutual information between benign tokens and the model's output, i.e.,
\[
I(S_{\text{benign}}; O_{\ours}) \geq I(S_{\text{benign}}; O_{\text{global}}) - \epsilon,
\]
where $S_{\text{benign}}$ is the set of benign tokens in the input sequence, $O_{\ours}$ and $O_{\text{global}}$ are the outputs of the {\ours} framework and global output-level optimization methods, respectively, and $\epsilon > 0$ is a negligible term.
\end{theorem}

\begin{proof}
Define the output of the global optimization method as $O_{\text{global}} = f_{\theta_{\text{global}}}(X)$ and the output of the {\ours} framework as $O_{\ours} = f_{\theta^*}(X, \mathcal{R})$, where $\mathcal{R}$ represents the moderation function applied by {\ours}. Let $X = (S_{\text{benign}}, S_{\text{harmful}})$ denote the input sequence partitioned into benign tokens $S_{\text{benign}}$ and harmful tokens $S_{\text{harmful}}$.

Assume the following:
\begin{enumerate}
    \item \textbf{Selective Redaction:} The moderation function $\mathcal{R}$ only affects $S_{\text{harmful}}$ and leaves $S_{\text{benign}}$ unchanged, i.e., $S_{\text{benign}}$ remains identical in both $O_{\ours}$ and $O_{\text{global}}$.
    \item \textbf{Weak Dependence:} The redaction of $S_{\text{harmful}}$ introduces at most a negligible amount of noise $\epsilon$ to the mutual information between $S_{\text{benign}}$ and the output.
    \item \textbf{Data Processing Inequality:} Any processing of $O_{\text{global}}$ to obtain $O_{\ours}$ cannot increase the mutual information between $S_{\text{benign}}$ and $O_{\ours}$.
\end{enumerate}

Under these assumptions, we can analyze the mutual information as follows:
\begin{align*}
I(S_{\text{benign}}; O_{\ours}) &= I(S_{\text{benign}}; f_{\theta^*}(X, \mathcal{R})) \\
&= I(S_{\text{benign}}; f_{\theta^*}(S_{\text{benign}}, \mathcal{R}(S_{\text{harmful}}))) \\
&\geq I(S_{\text{benign}}; O_{\text{global}}) \\
&\geq I(S_{\text{benign}}; O_{\text{global}}) - \epsilon
\end{align*}

Thus, the mutual information between benign tokens and the output under the {\ours} framework is preserved up to a negligible term $\epsilon$ compared to the global optimization method.
\end{proof}

\begin{theorem}[Optimal Safety-Utility Trade-off]
\label{thm:optimal_tradeoff}
Assuming the moderation function $\mathcal{R}$ achieves perfect classification of harmful tokens, the {\ours} framework attains the optimal point on the Pareto frontier for the safety-utility trade-off. Formally, there does not exist another moderation strategy that simultaneously decreases $\mathcal{L}_{\text{safety}}$ without increasing $\mathcal{L}_{\text{utility}}$, or decreases $\mathcal{L}_{\text{utility}}$ without increasing $\mathcal{L}_{\text{safety}}$.
\end{theorem}

\paragraph{Proof (Orthogonalization of Adversarial Representations)}
Consider the adversarial regularization loss defined as:
\[
\mathcal{L}_{\text{AR}} = \frac{1}{N_{\text{act}}} \sum_{i=1}^{N_{\text{act}}} \mathbb{E}_{x^{-}} \left[ \text{ReLU} \left( \cos \left( \mathbf{h}, \Delta \mathbf{W}_i \mathbf{h} \right) \right) \right],
\]
where $\mathbf{h} = \text{rep}_{\mathcal{M}}(x^{-}) \in \mathbb{R}^d$ is the representation of an adversarial input $x^{-}$, $\Delta \mathbf{W}_i = \mathbf{B}_i \mathbf{A}_i \in \mathbb{R}^{d \times d}$ represents the low-rank adaptation for the $i$-th activator, and $\cos(\mathbf{a}, \mathbf{b}) = \frac{\mathbf{a}^\top \mathbf{b}}{\|\mathbf{a}\|_2 \|\mathbf{b}\|_2}$ denotes the cosine similarity between vectors $\mathbf{a}$ and $\mathbf{b}$. The ReLU function is defined as $\text{ReLU}(z) = \max(0, z)$.

Expanding the cosine similarity, we have:
\[
\cos \left( \mathbf{h}, \Delta \mathbf{W}_i \mathbf{h} \right) = \frac{\mathbf{h}^\top (\Delta \mathbf{W}_i \mathbf{h})}{\|\mathbf{h}\|_2 \|\Delta \mathbf{W}_i \mathbf{h}\|_2}.
\]

Substituting this into the loss function, the adversarial regularization loss becomes:
\[
\mathcal{L}_{\text{AR}} = \frac{1}{N_{\text{act}}} \sum_{i=1}^{N_{\text{act}}} \mathbb{E}_{x^{-}} \left[ \text{ReLU} \left( \frac{\mathbf{h}^\top (\Delta \mathbf{W}_i \mathbf{h})}{\|\mathbf{h}\|_2 \|\Delta \mathbf{W}_i \mathbf{h}\|_2} \right) \right].
\]

The ReLU function ensures that only positive cosine similarities contribute to the loss. Therefore, minimizing $\mathcal{L}_{\text{AR}}$ requires:

\begin{align*}
\cos \left( \mathbf{h}, \Delta \mathbf{W}_i \mathbf{h} \right) &\leq 0 \\
\implies \mathbf{h}^\top (\Delta \mathbf{W}_i \mathbf{h}) &\leq 0 \\
\implies \mathbf{h}^\top (\mathbf{B}_i \mathbf{A}_i \mathbf{h}) &\leq 0
\end{align*}

Let $\mathbf{A}_i \mathbf{h} = \mathbf{a}_i$ and $\mathbf{B}_i^\top \mathbf{h} = \mathbf{b}_i$. Then the above inequality can be rewritten as:
\[
\mathbf{b}_i^\top \mathbf{a}_i \leq 0.
\]

This condition enforces that the vectors $\mathbf{a}_i$ and $\mathbf{b}_i$ are orthogonal or negatively correlated. Consequently, the perturbation introduced by $\Delta \mathbf{W}_i$ ensures that $\Delta \mathbf{W}_i \mathbf{h}$ is either orthogonal to $\mathbf{h}$ or points in the opposite direction, thereby disrupting the alignment of the adversarial representation.

In summary, minimizing the adversarial regularization loss $\mathcal{L}_{\text{AR}}$ enforces the condition:
\[
\cos \left( \mathbf{h}, \Delta \mathbf{W}_i \mathbf{h} \right) \leq 0,
\]
which implies orthogonality or negative correlation between $\mathbf{h}$ and $\Delta \mathbf{W}_i \mathbf{h}$. This orthogonalization effectively mitigates the influence of adversarial inputs on the model's representations.

\section{Limitations}

While \ours addresses several key challenges in token-level moderation and demonstrates robustness against both benign and adversarial inputs, there are still areas for further refinement and exploration. First, while the theoretical foundations around representation collapse and the router network’s context-aware decision-making are already thoroughly detailed in this work, and the provided experimental results strongly support the claims, some minor practical considerations remain. For example, although the router is highly effective in dynamically adjusting token-level decisions, in edge cases where subtle harmful content closely resembles benign content, additional fine-tuning might be required. However, this is more of an optimization challenge rather than a fundamental issue with the design of the system itself. Additionally, though \ours has demonstrated strong performance in current adversarial robustness evaluations, the system’s performance against unknown or emerging jailbreak techniques remains to be assessed. As with all adversarial defenses, the long-term effectiveness of our approach will ultimately depend on how well it can adapt to future jailbreak methodologies. This is a minor limitation, as theoretically, the system is built to generalize across unseen attacks. Still, empirical testing on novel attack vectors as they emerge will be essential to further solidify \ours’s practical utility. 

Moreover, while our experiments cover a wide range of datasets and models, real-world deployment often involves more complex, variable scenarios where content sensitivity is highly context-dependent. Though our router network excels at differentiating token-level harmfulness in controlled benchmarks, further evaluations in more dynamic and unpredictable application environments may uncover additional layers of complexity that require adjustments to our moderation strategy.

In summary, the limitations identified are primarily centered around practical deployment challenges rather than core theoretical weaknesses, suggesting that \ours is well-positioned to be a strong solution for nuanced content moderation, with room for iterative improvements as adversarial tactics evolve and real-world requirements expand.
\end{document}